\documentclass[sigconf]{acmart}

\renewcommand\footnotetextcopyrightpermission[1]{} 

\usepackage{graphicx}
\usepackage{balance}  
\usepackage{epsfig}
\usepackage{graphicx}
\usepackage{epstopdf}
\usepackage{latexsym}
\PassOptionsToPackage{noend}{algpseudocode}
\usepackage{algpseudocode}
\usepackage{pifont}
\usepackage{epsfig}
\usepackage{amsmath}
\usepackage{amsfonts}
\usepackage{subfigure}
\usepackage{stmaryrd}
\usepackage{url}
\usepackage{multirow}
\usepackage{array}
\usepackage[normalem]{ulem}
\usepackage{color}
\usepackage{cleveref}
\usepackage{xspace}
\usepackage{mathtools}
\usepackage{soul}
\usepackage{listings}
\usepackage{enumitem}
\usepackage{appendix}
\usepackage{hyperref}
\usepackage{multicol}

\usepackage{algorithm}
\usepackage{algorithm}
\usepackage{algorithmicx}
\usepackage{algpseudocode}
\usepackage{amsmath}
\usepackage{verbatim}

\newlist{compactitem}{itemize}{3} 
\setlist[compactitem]{label=\textbullet, nosep, leftmargin=0cm,itemindent=.5cm}

\usepackage{bm}

\theoremstyle{definition}

\makeatletter
\def\algbackskip{\hskip-\ALG@thistlm}
\makeatother

\newcommand{\sys}{RITA\xspace}
\newcommand{\system}{RITA\xspace}

\newcolumntype{M}[1]{>{\centering\arraybackslash}m{#1}}
\errorcontextlines\maxdimen

\begin{document}
\pagestyle{plain}
\title{\sys: Group Attention is All You Need for Timeseries Analytics}

 \author{Jiaming Liang}
 \affiliation{%
   \institution{University of Pennsylvania}
   \city{Philadelphia}
   \state{PA}
   \country{USA}
 }
 \email{liangjm@seas.upenn.edu}

 \author{Lei Cao}
 \authornote{Corresponding Author}
 \affiliation{%
   \institution{Massachusetts Institute of Technology}
   \city{Cambridge}
   \state{MA}
   \country{USA}
 }
 \email{lcao@csail.mit.edu}

 \author{Samuel Madden}
 \affiliation{%
   \institution{Massachusetts Institute of Technology}
   \city{Cambridge}
   \state{MA}
   \country{USA}
 }
 \email{madden@csail.mit.edu}

 \author{Zachary Ives}
 \affiliation{%
   \institution{University of Pennsylvania}
   \city{Philadelphia}
   \state{PA}
   \country{USA}
 }
 \email{zives@cis.upenn.edu}

 \author{Guoliang Li}
 \affiliation{%
   \institution{Tsinghua University}
   \city{Beijing}
   \country{China}
 }
 \email{liguoliang@tsinghua.edu.cn}

\begin{abstract}
Timeseries analytics is of great importance in many real-world applications. Recently, the Transformer model, popular in natural language processing, has been leveraged to learn high quality feature embeddings from timeseries, core to the performance of various timeseries analytics tasks. However, the quadratic time and space complexities limit Transformers' scalability, especially for long timeseries. To address these issues, we develop a timeseries analytics tool, \system, which uses a novel \emph{attention} mechanism, named \emph{group attention}, to address this scalability issue.
Group attention dynamically clusters the objects based on their similarity into a small number of groups and approximately computes the attention at the coarse group granularity. It thus significantly reduces the time and space complexity, yet provides a theoretical guarantee on the quality of the computed attention.
The dynamic scheduler of \system continuously adapts the number of groups and the batch size in the training process, ensuring group attention always uses the fewest groups needed to meet the approximation quality requirement. 
Extensive experiments on various timeseries datasets and analytics tasks demonstrate that \system outperforms the state-of-the-art in accuracy and is significantly faster --- with speedups of up to 63X. 
\end{abstract}
\maketitle




\settopmatter{printacmref=false} 
\renewcommand\footnotetextcopyrightpermission[1]{} 




\section{Introduction}
\label{sec.intro}
\noindent\textbf{Motivation.} Many data driven applications involve processing massive timeseries data, including IoT~\cite{cook2019anomaly}, medical AI~\cite{crabtree1990individual}, stock market~\cite{kraft1977determinants}, and so on. As such, there is a great need for timeseries analytics, such as forecasting~\cite{chatfield2000time}, classification~\cite{ismail2019deep}, clustering~\cite{liao2005clustering}, similarity search~\cite{negi2005time}, and anomaly detection~\cite{teng2010anomaly}, with applications ranging from automatically diagnosing diseases~\cite{bui2017time}, recognizing human activities~\cite{lara2012survey}, to stopping financial fraud~\cite{yue2007review}.


Effective feature extraction~\cite{paparrizos2019grail} lies at the core of almost all these timeseries analytics tasks. Recently researchers~\cite{DBLP:conf/kdd/ZerveasJPBE21} have started leveraging the {\it self-supervised pre-training} methodology of Transformers~\cite{DBLP:conf/nips/VaswaniSPUJGKP17,DBLP:conf/naacl/DevlinCLT19,brown2020language}, which have proven remarkably successful in natural language processing (NLP), to automatically learn high quality feature embeddings from timeseries. In NLP, self-supervised pre-training exploits the sequential patterns (correlations) among the words in sentences to produce {\it contextualized} feature embeddings. Timeseries bear similarity to natural language, because in timeseries data the sequential order among the values (stock  price, volume, etc.) over time matters. That is, each value is highly correlated with other values observed before or after it. Therefore, pre-training a Transformer model which takes the correlations among different observations into account is a natural idea to learn feature embeddings from timeseries. Indeed, the experiments in \cite{DBLP:conf/kdd/ZerveasJPBE21} confirm that Transformer-based methods outperform traditional timeseries analytics techniques.
 

However, existing work~\cite{DBLP:conf/kdd/ZerveasJPBE21} that directly applies Transformers to learn features from timeseries data have been shown not to be scalable to {\it long} timeseries~\cite{li2019enhancing}. The idea of self-attention~\cite{DBLP:conf/nips/VaswaniSPUJGKP17} is central to pre-training methods in NLP: It computes pairwise correlations among different semantic units in a sequence (in NLP, a sentence); as such, it has {\it quadratic time and space} complexity in the length of the input sequence. 
Such an approach places limits on the model's scalability, especially when handling large sequences,  which are common in real-world timeseries applications such as IoT, medical AI, and finance~\cite{zhou2021informer,DBLP:journals/pvldb/CaoTAJYLGSBSCWM19,liu2018open}. Predictions about timeseries may need to look at months or years of historical data to make accurate predictions, spanning hundreds of thousands of samples.  
As an example, in collaboration with a research hospital we have been developing a seizure classifier that automatically detects seizures based on EEG signals (timeseries) collected during the clinical observation of patients. As seizures last only a few seconds, we chunk long EEG data into many 2 second segments and detect seizures at a segment level. However, the classification of a particular segment depends on up to 12 hours of prior signal to determine if one 2 second segment indicates seizure or not, because seizure diagnosis needs to consider long-term trends in the EEG data~\cite{DBLP:journals/pvldb/CaoTAJYLGSBSCWM19}. The number of segments in 12 hours is more than 21k. 
This is far larger than the number of semantic units the typical NLP tasks expect. For example, BERT~\cite{DBLP:conf/naacl/DevlinCLT19} limits the number of units to 512 and even massive models like GPT-3~\cite{brown2020language} limit the number of units to 2048.

Although in NLP some lower-complexity methods have been proposed to {\it approximately} compute self-attention~\cite{kitaev2020reformer,choromanski2020rethinking,wang2020linformer}, their performance degrades dramatically when used on timeseries, due to the gap between natural language and timeseries, as we will show in our experiments.

\noindent\textbf{Proposed Approach.} 
To tackle the aforementioned problem, we develop \textbf{RITA}, a Transformer-based timeseries analytics tool, which uses a novel attention mechanism, called {\bf group attention}, to scale to long timeseries.

Leveraging the periodicity of timeseries, \system chunks the input timeseries into segments and dynamically clusters the segments into a small number (denoted as $N$) of groups. Segments in the same group possess similar feature embeddings during the current training iteration, thus enabling them to approximately share the computation of attention. As the timeseries increases in length, more sharing opportunities become available. \system then computes the self-attention at a group level and produces a {\it compressed group attention matrix}.
In this way, group attention eliminates both computation and memory bottlenecks in Transformer-style models and thus more scalable to long timeseries.

However, making this idea effective and efficient in Transformer architectures is {\it challenging} for several reasons:

\begin{compactitem}
\item \textbf{Efficiently Producing High Quality Feature Embeddings.} Although \system computes the attention matrix at a group level, to preserve the quality of the feature embeddings, it still has to produce different embeddings for different segments. This is because even if some segments share the attention score temporally, it does not mean they should have the same feature embedding. However, using the group attention matrix, the existing self-attention mechanism will only produce a single feature vector for each group. A naive solution would be to restore the original attention matrix from the group attention matrix. 
However, in this case we again get an attention matrix with quadratic space complexity. Because GPUs have limited memory, GPU memory will remain a {\it bottleneck} in group attention. 


\item \textbf{The Number of Groups N.} In \system, the number of groups $N$ is a crucial factor that balances the speed up and the quality of attention approximation. A small $N$ will lead to a large speedup, but the approximation errors can also be significant. On the other hand, although a large $N$ tends to produce high-quality approximations, it inevitably slows down the training process. Therefore, an appropriate $N$ is essential to the performance of group attention. However, $N$ depends on the distributional properties of the dataset. Furthermore, like the classical transformer models, \system stacks multiple attention layers to produce better embeddings. Ideally, different layers should also use different values of $N$. In addition, during the model training phrase, group attention should use different values of $N$ at different iterations to adapt to the varying feature embeddings.
This makes manually setting appropriate $N$ almost impossible. 

\item \textbf{Batch Size.} Moreover, as we want to dynamically adjust $N$ during training, a fixed batch size is sub-optimal: as $N$ decreases, the memory usage of a single sample decreases. This allows a larger batch size which is beneficial, because: 
(1) it makes full use of GPU memory; (2) high-parallelism across the samples in a big batch brings better performance. 
Our experimental study shows that doubling the batch size reduces the training time by 30\%, while still preserving the quality of the model. Thus, \system should dynamically adjust batch size as $N$ changes.
\end{compactitem}





%
To address the above problems, we first propose an {\it embedding aggregation} strategy and a customized {\it group softmax function} to replace the classical softmax function~\cite{DBLP:conf/nips/VaswaniSPUJGKP17}. Together they ensure \system is able to directly use the compressed attention matrix to produce different feature embeddings for different segments. We theoretically show the embeddings \system produces in this way are identical to those produced by first re-storing the original large attention matrix. Thus \system is able to produce high quality embeddings without introducing extra overhead.
Further, we design a GPU friendly algorithm to group the segments {\it in parallel}, effectively minimizing the grouping cost.


Second, we design an {\it adaptive scheduler} which dynamically decides an appropriate $N$ for each group attention layer during the training process. It starts with a large $N$ and iteratively merges groups that are similar to each other. Guided by an error bound on the approximated self-attention that users can tolerate, it automatically determines if two groups are mergeable, performing merging efficiently in a GPU-friendly way.


Moreover, we propose a {\it learning-based method} to model the correlation between the number of groups $N$ and the batch size $B$. This model is used to predict $B$ for a given $N$ when training \system.
Specifically, we first sample some $N$ values in a reasonable range. For each sampled $N$, we find a batch size that consumes up to a certain percentage of GPU memory in a cost-efficient way. Using a small set of mathematical functions as a prior, \system learns a model with only a few <N, B> pairs as ground truth labels.    

Our experiments on public timeseries benchmarks and the MGH EEG data~\cite{DBLP:journals/pvldb/CaoTAJYLGSBSCWM19} confirm that \system outperforms  state-of-the-art methods in accuracy on various timeseries analytics tasks, while our group attention mechanism achieves a 63X speedup with much less memory required, compared to existing self-attention mechanisms~\cite{DBLP:conf/nips/VaswaniSPUJGKP17,choromanski2020rethinking,wang2020linformer}.

\noindent\textbf{Contributions.} The key contributions of this work include:

\begin{compactitem}


\item Our group attention mechanism leverages the periodicity of timeseries, reducing the time and space complexity of the self-attention mechanism with accuracy guarantees, allowing \system to scale to long timeseries data.

\item Guided by an approximation error bound, our adaptive scheduler dynamically adapts the number of groups and the batch size to the distribution properties of the evolving feature embeddings, making group attention efficient and easily tunable.

\item We conduct experiments on various datasets and different analytics tasks, demonstrating that \system is 4 to 63 times faster than the state-of-the-art while achieving better accuracy when handling long timeseries (length $\geq$ 2000).


\end{compactitem}

\section{Background}
\label{sec.preliminary}
\begin{sloppypar}

\begin{figure}[t]
 \vspace{-3mm}
     \centering
     \includegraphics[width=0.8\columnwidth]{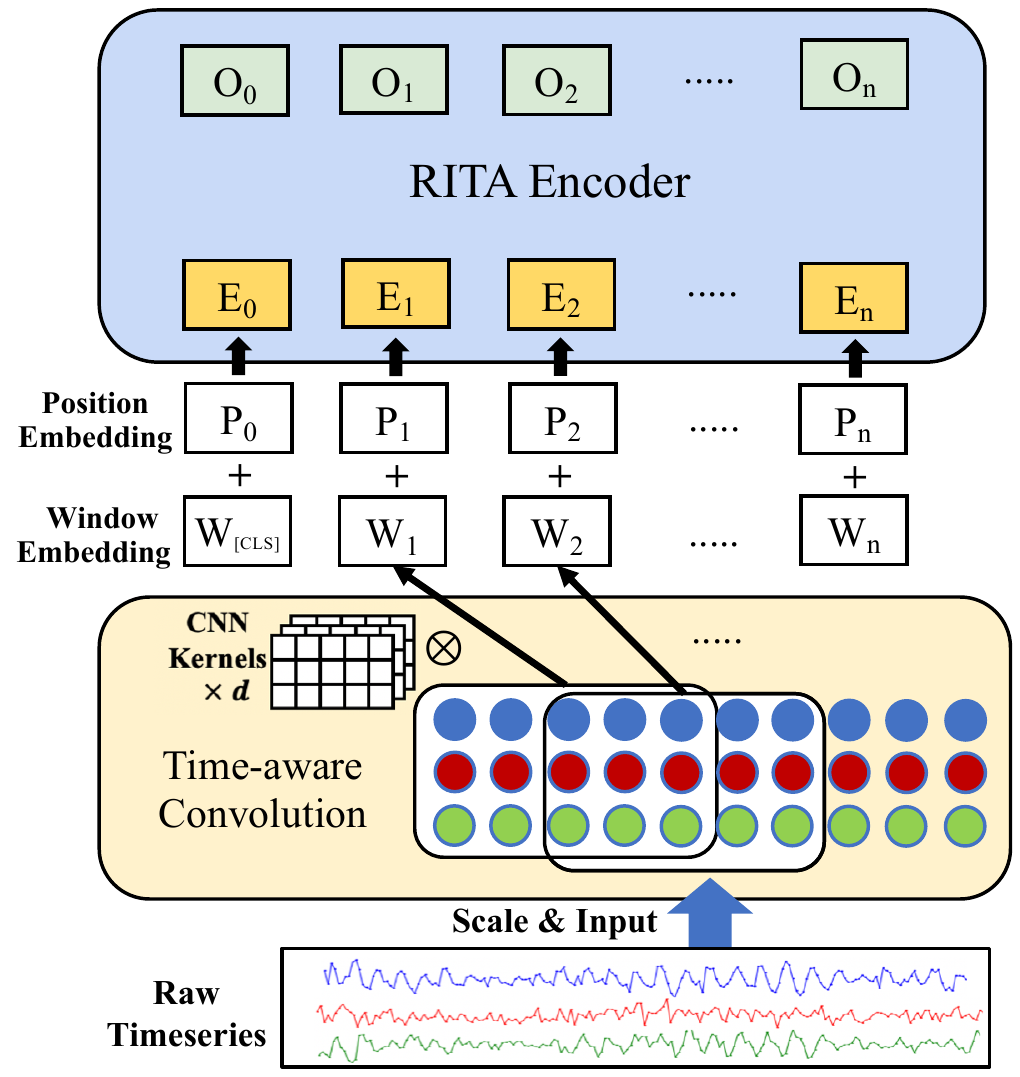}
     \vspace{-2mm}
     \caption{\system Architecture}
     \label{fig.convolution}
     \vspace{-4mm}
\end{figure}


We provide some background on the canonical self-attention module in the Transformer\cite{DBLP:conf/nips/VaswaniSPUJGKP17}.
A \emph{self-attention} module takes $n$ hidden embedding vectors $H \in \mathbb{R}^{n*d_h}$ as input, then projects them to queries ($Q$), keys ($K$) and values ($V$) and performs Scaled-dot Product Attention, which given input hidden state $H$, is computed by:

\vspace{-0.5mm}
\begin{equation}
\label{eq.attention}
\small
\begin{aligned}
    Q = HW_Q, K = HW_K, V = HW_V \\
    O = AV = SoftMax(\frac{QK^T}{\sqrt{d_k}})V 
\end{aligned}
\end{equation}
\vspace{-1mm}
Where $W_Q \in \mathbb{R}^{d_h*d_k}, W_K \in \mathbb{R}^{d_h*d_k}, W_V \in \mathbb{R}^{d_h*d_v}$ are projection matrices for generating $Q,K,V$.
$Q\in \mathbb{R}^{n*d_k}$ is also regarded as the packing of $n$ query vectors $\{q_1,...,q_n\}$ with dimension $d_k$ into a matrix. $K \in \mathbb{R}^{n*d_k}, V\in \mathbb{R}^{n*d_v}$ are regarded as the packing of key vectors $\{k_1,...,k_n\}$ and value vectors $\{v_1,...,v_n\}$ in the same way.

Given a matrix $M \in \mathbb{R}^{L*n}$, the softmax function normalizes $M$ to ensure the sum of each row equals to 1, as shown below.
\vspace{-2mm}
\begin{equation}
\label{eq.softmax}
SoftMax(M_{i,j})=\frac{exp(M_{i,j})}{\sum_{k=0}^{n-1} exp(M_{i,k})}\\
\end{equation}

Note the attention matrix A is an $n \times n$ matrix, where $n$ represents the number of elements in the input sequence (e.g. words in NLP).


\end{sloppypar}


\section{\system overview}
\label{sec.rita}
Given a collection of {\it unlabeled} timeseries, \system first pre-trains a Transformer-style model to produce high quality feature embeddings for timeseries data. This pre-trained model is then used to support various downstream tasks, similar to BERT~\cite{DBLP:conf/naacl/DevlinCLT19}. Next, we overview the model architecture of \system. We show how \system supports various downstream tasks in Appendix~\ref{appendix.downstream}. 

\begin{sloppypar}

As shown in Fig.~\ref{fig.convolution},
RITA is consist of two components: (1) Time-aware Convolution Layer (2) RITA Encoder.

\noindent\textbf{Time-aware Convolution Layer} fills the gap between timeseries and natural language. Despite their high-level similarity, there is a big gap between timeseries and natural language. 
First, in natural language each word, as a discrete semantic unit, has an independent meaning, while each element in a timeseries is a continuous, numerical value and does not necessarily constitute an independent event. 
Furthermore, the input sequences are single-channeled in NLP, but often multi-channeled in timeseries (i.e., sensor data often consists of several related channels).

\system leverages the classical convolution~\cite{NIPS2012_c399862d} strategy to solve this problem. Convolution is widely used to capture the local structures of an image. We use convolution to chunk one input timeseries into a sequence of windows and learn the local structure of each window, similar to the discrete semantic units in natural language. 
It also discovers the correlations across different channels, thus naturally solving the multi-channel problem. 

More specifically, treating a multi-variate timeseries of length $n$ and with $m$ variables as an $\mathit{n \times m}$ matrix $T$, \system uses $d$ convolution kernels to chunk $T$ into \textbf{n} windows and produce one d-dimensional embedding per window using the convolution operation~\cite{NIPS2012_c399862d}. Each convolution kernel corresponds to a $\mathit{w \times m}$ matrix, where $w$ defines the number of timestamps that each convolution kernel covers, identical to the window size in sliding window. 




\noindent\textbf{RITA Encoder} functions as Transformer Encoder as described in the original Transformer work\cite{DBLP:conf/nips/VaswaniSPUJGKP17}. It takes the embeddings of $n$ semantic units $X_1,X_2,...,X_n (X_i \in R^d)$ as input (e.g. embeddings of $n$ windows for a timeseries), then models the correlations between the semantic units and outputs $Y_1,...,Y_n (Y_i \in R^d)$ as the context-aware embedding of each unit. 

What makes RITA Encoder different from Transformer Encoder is that: at the core of Transformer Encoder lies self-attention mechanism which incurs a $O(n^2)$ time complexity and memory usage. This quadratic cost becomes prohibitive for long timeseries and limits the scalablity of Transformer-based models. To make the attention computation efficient yet high-quality, we replace the canonical self-attention with our proposed {\bf group attention}.

\noindent\textbf{Self-supervised Pretraining.}
Inspired by the ``cloze text'' pretraining task in NLP, we designed a mask-and-predict task as the pretraining task for our model. The timeseries is randomly masked and the model should recover the masked values based on corresponding contextual information.

To be specific, we generate masks on time-stamps, with a mask rate $p$. The timeseries is scaled to be non-negative and the values across all the channels on the masked timestamps are set to be -1, an impossible value on normal timestamps. Then the masked timeseries is fed into \system and the output representation is translated to the recovered timeseries by a Transpose Convolution layer.

\end{sloppypar}

\section{Group attention mechanism}
\label{sec.group}

\begin{sloppypar}
Group attention, a novel and efficient approximate attention mechanism, addresses the performance bottleneck of self-attention in the vanilla Transformer.
In this section, we first introduce the framework of group attention and then theoretically establish the bound of its approximation error.

\subsection{The Idea of Group Attention}
\label{sec.group.ideas}

As periodicity is a natural property of timeseries~\cite{10.1145/3448016.3452779}, similar windows frequently occur. Similar windows result in similar queries/keys for attention computation, bringing opportunities for saving computation.


As discussed in Sec.~\ref{sec.preliminary}, $A_{ij}$, the attention score of window $i$ onto window $j$, is determined by the inner product between the query vector of window $i$ and the key vector of window $j$, that is, $q_i \cdot k_j$. 
Given another window $x$, if window $x$ has the similar key vector to window $j$, that is, $k_j$ $\approx$ $k_x$, then $q_i \cdot k_j$ $\approx$ $q_i \cdot k_x$.
In other words, $A_{ij}$ $\approx$ $A_{ix}$ when $k_j$ $\approx$ $k_x$.


This observation inspires our group attention mechanism. That is, we group the windows by their similarity in keys. 
Assuming all windows in the same group have the same attention score onto another window $k$, we then only compute the attention once by using {\it one single key} to represent this group, for example the centroid of the group of keys.
This thus saves significant computation cost.

Better yet, after grouping $n$ windows into $N$ groups, group attention compresses the attention matrix from an $n \times n$ matrix to an $n \times N$ matrix. Because $N$ (number of groups) tends to be much smaller than $n$ (number of windows) due to the periodicity of timeseries, group attention consumes much less memory than the original self-attention mechanism, successfully eliminating the memory bottleneck.  Note that it also doesn't hurt quality all that much, as confirmed in our experiments (Sec.~\ref{sec.exp.effective}).









\begin{figure}[]
\vspace{-3mm}
    \centering    \includegraphics[width=1.05\columnwidth]{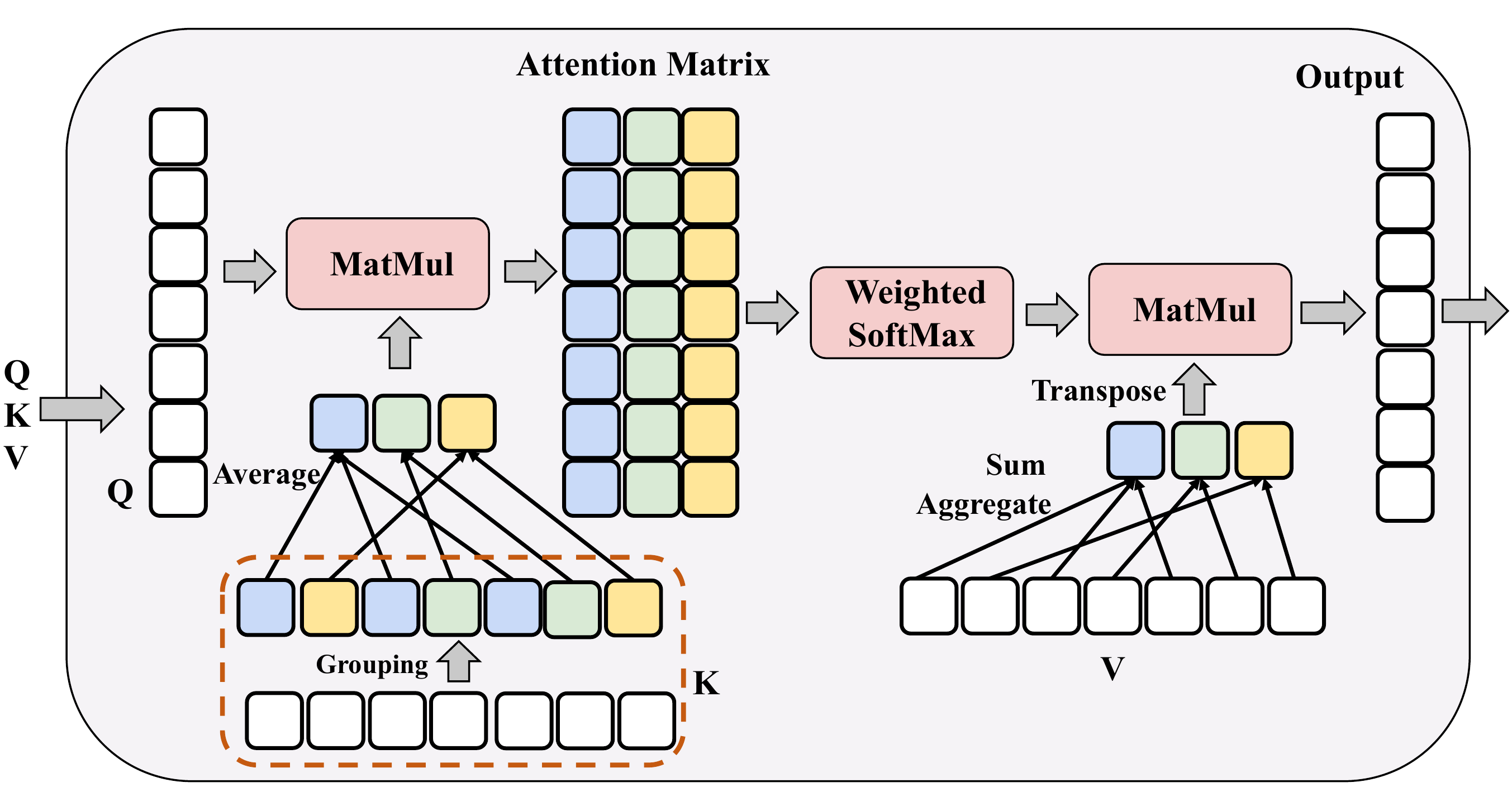}
    \vspace{-7mm}
    \caption{Group Attention}
    \label{fig.group}
    \vspace{-5mm}
\end{figure}

\subsection{Computing the Output Feature Embedding}
\label{sec.group.embedding}
We now discuss how to efficiently compute the output feature embeddings using the small compressed group attention matrix.

\vspace{-1mm}
\subsubsection{Problem: Producing Embeddings w/ Group Attention Matrix\nopunct}\ \\
As described in the Background, once we have acquired the attention matrix $A$, canonical self-attention computes the output embedding $O$ as $\mathit{O = AV}$. Because $A$ is an $n \times n$ matrix and $V$ is an $n \times d_v$ matrix, the matrix product operation still produces an $n \times d_v$ matrix $O$. That is, it produces a $d_v$ dimensional feature vector for each {\it window}.
However, our group attention will produce an $n \times N$ attention matrix $\widetilde{A}$ , where $N$ corresponds to the number of groups. 
In this case the matrix product will produce a $N \times d_v$ matrix $\widetilde{O}$. That is, it produces a feature vector for each {\it group}. 
However, our goal is to produce different embeddings for different windows, because even if some windows share the attention score temporally, it does not mean they should have the same feature embedding. 

\noindent\textbf{A Naive Solution.} A naive solution would be to restore the full attention matrix $A$ from the group attention matrix $\widetilde{A}$. For example, given one group composed of $win_i$ and $win_j$, we map its group attention vector in  $\widetilde{A}$ into two rows that correspond to $win_i$ and $win_j$ in $A$. 
However, in this case we again get a $n \times n$ attention matrix; and GPU memory remains a {\it bottleneck} in group attention.

\vspace{-1mm}
\subsubsection{Solution: Embedding Aggregation and Group SoftMax\nopunct}\ \\
\label{sec.group.efficient}
Using an \textit{embedding aggregation} operation and a \textit{group softmax} function, \system produces $n$ embeddings without restoring the full attention matrix. Fig.~\ref{fig.group} shows the workflow of group attention. 

\noindent\textbf{Embedding Aggregation.} The idea is inspired by the observation on the matrix product operation $\mathit{O = AV}$ conducted on the fully restored attention matrix $A$. 

Given an element $O_{i,j}$ of $O$ corresponding to the $j^{th}$ dimension of $win_i$'s feature vector, $O_{i,j}$ = $a_i \cdot v_j$, where vector $\mathit{a_i \in \mathbb{R}^{n}}$ denotes the $i^{th}$ row of the attention matrix $A$ and vector $\mathit{v_j \in \mathbb{R}^{n}}$ denotes the $j^{th}$ dimension of all the $n$ feature vectors. 
Given $\mathit{a_i = <a_i^1, a_i^2, \cdots, a_i^n>}$ and $\mathit{v_j = <v_j^1, v_j^2, \cdots, v_j^n>}$, $O_{i,j}$ = $\mathit{\sum_{k=1}^n a_i^k v_j^k}$.

As an example, assume $win_1$ and $win_2$ belong to the same group $G_1$. Then $a_i^1$ = $a_i^2$ = $\widetilde{a}_i^1$, where $\widetilde{a}_i^1$ $\in$ $\widetilde{A}$ corresponds to the attention of group $G_1$ onto $win_i$. 
Therefore, $a_i^1 v_j^1$ + $a_i^2 v_j^2$ = $\widetilde{a}_i^1$ ($v_j^1$ + $v_j^2$).

As an immediate generalization of the above analysis, if we aggregate up the windows that belong to the same group and convert the n-dimensional feature vector $v_j$ into a $N$-dimensional group feature vector $\widetilde{v}_j$ beforehand, we could directly use the group attention vector $\widetilde{a}_i$ and the group feature vector $\widetilde{v}_j$ to compute $O_{i,j}$.

Using embedding aggregation, \system is able to produce the feature embedding $\widetilde{O}$ that is identical to the embedding $O$ produced by using the full attention matrix $A$ and the embedding matrix $V$.

\noindent\textbf{Group Softmax Function.}
In canonical self-attention the attention matrix $A$ is computed as $A$ = $\mathit{SoftMax(\frac{QK^T}{\sqrt{d_k}})}$. To compute $A$, we have to first compute $QK^T$ (denoted as $P$) which is an $n \times n$ matrix. Then normalizing the $P$ matrix with softmax produces the attention matrix $A$. 

Group attention follows the same procedure. But after grouping keys into $\widetilde{K}$, $Q\widetilde{K}^T$ produces an $n \times N$ matrix $\widetilde{P}$. Due to the non-linearity of the softmax function, applying softmax directly on $\widetilde{P}$ will result in a group attention matrix $\widetilde{A}$ from which we are not able to recover a full attention matrix that is identical to first restoring $\widetilde{P}$ to $P$ and then applying softmax on $P$. The $A$ matrix produced by the latter is desirable, as we want to approximate the original attention matrix as accurately as possible. 
However, restoring the small $n \times N$ $\widetilde{P}$ matrix is not memory efficient, as it will end up with a full $n \times n$ matrix $P$. 

To solve the above problems, we introduce a new \textbf{group softmax} function to replace the original softmax function (Eq.~\ref{eq.softmax}).

\vspace{-2mm}
\begin{equation}
\label{eq.groupSoftmax}
GroupSoftMax(\widetilde{P_{i,j}}) = \frac{exp(P_{i,j})}{\sum_{k=0}^{N-1} count_k exp(P_{i,k})}\\
\end{equation}

In Eq.~\ref{eq.groupSoftmax}, $count_k$ represents the number of windows that Group $G_k$ contains. Compared to the original softmax, our group softmax considers each group $G_k$ as $count_k$ elements and counts it $count_k$ times when summing up the exponential of each group's $P_{i,k}$.
In this way, the group softmax function operating on the small $\widetilde{P}$ matrix will produce {\it exactly the same} result to the softmax function operating on the full $P$ matrix.

\noindent\textbf{Theoretical Guarantee.} In Appendix~\ref{appendix.proof.groupAttention}, we prove that the group softmax function and the embedding aggregation operation produce the same output feature embedding with the naive method that has to first restore the big full attention matrix.

We show an efficient implementation of the embedding aggregation operation and group softmax function in Appendix~\ref{appendix.groupAttention}, Alg.~\ref{algo.grpattn}.

\noindent\textbf{Time Complexity.}
The time complexity of Alg.~\ref{algo.grpattn} is $O(nNd)$ and the space complexity is $O(nN)$, while the time and space complexity of the original self-attention mechanism are $O(n^2d)$ and $O(n^2)$.

\subsection{Error Bound}
\label{sec.group.error}

Group attention produces a group attention matrix $\widetilde{A}$ which approximates the attention matrix $A$ produced by the classical self-attention with a {\it bounded error}, as shown in Lemma~\ref{lm.grperrorbound}.
\vspace{-1mm}
\begin{lemma}
\label{lm.grperrorbound}
Let $R$ be the radius of the ball where all key vectors live; $\widetilde{k}_i$ be the representative of the group that contains key $k_i$. Let $\overline{A}$ denote the full attention matrix restored from $\widetilde{A}$. Suppose the distance between $\widetilde{k}_i$ and $k_i$ $(||\widetilde{\mathbf{k}}_i-\mathbf{k}_i||)$ satisfies: $||\widetilde{\mathbf{k}}_i-\mathbf{k}_i|| \leq$ {\bf d}.

Then $\forall$ $\epsilon > 1$, if $\mathit{d \leq \frac{\ln(\epsilon)}{2R}}$, $\mathit{\frac{1}{\epsilon} \leq \frac{\overline{A}_{i,j}}{A_{i,j}} \leq \epsilon}$
\end{lemma}

Lemma~\ref{lm.grperrorbound} shows that the error bound $\epsilon$ of the group attention is determined by the distance $d$. 
As discussed in Sec.~\ref{sec.scheduler.group}, it inspires us to design a strategy to dynamically determine the number of groups $N$ -- the most critical parameter of group attention. 
Please refer to Appendix~\ref{appendix.proof.errorBound} for the proof. 

\subsection{GPU Friendly Grouping Method}
In this section, we discuss the implementation of a grouping method. To make group attention efficient and effective, the grouping method has to satisfy the following requirements: 

(1) Tight distance bound: to ensure the approximation quality, the distance between each key and its group representative should be minimized according to Lemma~\ref{lm.grperrorbound}.

(2) Lightweight: to ensure the performance gain, the grouping method must be lightweight, at worst not exceeding the complexity of group attention itself ($O(Nn)$).

(3) GPU friendly: to take advantage of GPUs, we prefer a grouping method that mainly consists of matrix operations, which can be efficiently executed on a GPU.

To satisfy the above requirements, after thorough investigation on various clustering algorithms, we design a GPU friendly K-means~\cite{lloyd1982least} as the grouping method. 

First, K-means minimizes the overall distance between any object and its cluster center, hence naturally satisfying Requirement 1. 

Second, given $N$ centers, in each iteration the time and space complexity of K-means is $O(nN)$. Usually, the iteration goes until convergence. However, we observe that rather than seeking a perfect K-means clustering, training a few iterations is sufficient to get a good grouping for group attention, because typically the later iterations only slightly update the clustering and group attention is robust to such imperfection. 

Third, we design a GPU-friendly implementation of K-means. The performance bottleneck of K-means comes from the distance computation between each vector and its center, that is, $\mathit{|v_i-c_j|=\sqrt{(v_i-c_j)^2}, i\in [1,n], j\in [1,N]}$. The performance bottleneck is $v_i-c_j$.
We instead use a different formulation: $|v_i-c_j|=\mathit{|v_i-c_j|=\sqrt{|v_i|^2+|c_j|^2-2 v_i \cdot c_j}, i\in [1,n], j\in [1,N]}$.
This is because in this formulation, the performance bottleneck is $v_i \cdot c_j$, which could be implemented as a matrix product operation.
Although the complexity of the two formulations is the same, in GPUs matrix product is much more efficient than pairwise difference.







\end{sloppypar}

\vspace{-2mm}
\section{Adaptive scheduler}
\label{sec.scheduler}
Next, we present the adaptive scheduler of \system which addresses the challenges of determining an appropriate number of groups $N$ and accordingly the batch size $B$, as described in Introduction. Using a dynamic scheduling method we propose, the scheduler automatically determines and adjusts $N$ and $B$ based on the distributional properties of the feature embeddings produced over the iterative training process, while guaranteed to produce high quality attention approximation that meets the requirement of users.

In Sec.~\ref{sec.scheduler.group} we show how \system automatically determines $N$. 
Then we introduce in Sec.~\ref{sec.scheduler.batch} the learning-based method which given an $N$, immediately predicts a good batch size.

\subsection{Dynamically Determining the Number of Groups N}
\label{sec.scheduler.group}
Without loss of generality, we use one group attention module as an example to show how \system automatically gets an appropriate $N$. The adaptive scheduler of \system starts with a large $N$ and decreases it dynamically. This is because in the training process of \system, the feature embeddings produced epoch by epoch tend to get stabler and stabler and gradually converge, thus no need to increase $N$.

\system reduces the number of groups by merging similar groups. Intuitively, given two groups, we could measure their similarity based on the distance of their centers. If the distance between their centers is smaller than a distance threshold, then the two groups could be merged. However, setting an appropriate distance threshold seems hard -- as difficult as setting an appropriate $N$.

To solve this problem, \system leverages the error bound of group attention introduced in Sec.~\ref{sec.group.error}. It only requires users to set an error bound $\epsilon$, and then uses Lemma~\ref{lm.grperrorbound} to translate $\epsilon$ to a distance threshold $d$.
\system then uses Lemma~\ref{lm.mergealbe} to determine if merging some given clusters still meets the error bound threshold $\epsilon$. 

\vspace{-2mm}
\begin{lemma}
\label{lm.mergealbe}
Denote $c_k$ to be the cluster center of $cluster_k$. Assume the existing grouping satisfies $\mathit{\forall k,\mathop{max}\limits_{x \in cluster_k} |c_k-x| \leq d}$
, thus satisfying an error bound $\epsilon$ by Lemma~\ref{lm.grperrorbound}.
If there exist $m$ clusters, namely, $cluster_{k_1},cluster_{k_2},...,cluster_{k_m}$, satisfying that:
\begin{equation}
\label{equ.mergecond}
\mathop{max}\limits_{x \in cluster_{k_i}} |c_{k_i}-c_{k_j}|+|x-c_{k_i}| \leq d, i,j \in [1,m]
\end{equation}
merging them into one cluster still meets the error bound $\epsilon$.
\end{lemma}
Please refer to Appendix~\ref{appendix.proof.merge} for the proof.

\noindent\textbf{Finding the Mergable Clusters.} We formulate the problem of finding mergeable clusters using graph theory: 

(1) each cluster is a node in the graph; 

(2) if $cluster_i$ and $cluster_j$ satisfy:

$\mathop{max}\limits_{x \in cluster_{i}} |c_{i}-c_{j}|+|x-c_{i}| \leq d$, and
$\mathop{max}\limits_{x \in cluster_{j}} |c_{j}-c_{i}|+|x-c_{j}| \leq d$

there is an undirected edge between $node_i$ and $node_j$;

In this scenario, finding the maximum number of \textit{mergeable} clusters is equivalent to finding the minimal clique cover in the corresponding graph, which is an NP-hard problem~\cite{karp1972reducibility}. 
Such heavy computation overhead is not acceptable for \system. We thus offer a simplified solution:

(1) Halve the clusters into two sets $S_1,S_2$;

(2) If $cluster_i \in S_1$ and $cluster_j \in S_2$ satisfy:
\vspace{-2mm}
\begin{equation}
\small
\label{eq.simplemerge}
\begin{aligned}
\mathop{max}\limits_{x \in cluster_{i}} |c_{i}-c_{j}|+|x-c_{i}| \leq d,  \mathop{max}\limits_{x \in cluster_{j}} |c_{j}-c_{i}|+|x-c_{j}| \leq \frac{d}{2}
\end{aligned}
\end{equation}

$cluster_j$ is marked.

(3) Decrease the number of clusters by counting the masks in $S_2$.

\smallskip
\noindent
In this solution, clusters in $S_1$ can be regarded as transfer nodes.
If \eqref{eq.simplemerge} holds for $(cluster_i \in S_1,cluster_{j_1}\in S_2)$ and $(cluster_i \in S_1,cluster_{j_2}\in S_2)$, respectively, we have,
\vspace{-1mm}
\begin{equation}
\small
\begin{aligned}
&\mathop{max}\limits_{x \in cluster_{j_1}} |c_{j_1}-c_{j_2}|+|x-c_{j_1}| \\
\leq &\mathop{max}\limits_{x \in cluster_{j_1}} |c_{j_1}-c_{i}|+|c_{i}-c_{j_2}|+|x-c_{j_1}|\\
\leq &\mathop{max}\limits_{x \in cluster_{j_1}} |c_{j_1}-c_{i}|+|c_{i}-c_{j_2}|+|x-c_{j_1}|+|x-c_{j_2}| \leq d
\end{aligned}
\end{equation}

Thus \eqref{equ.mergecond} holds when merging several clusters in $S_2$ with one cluster in $S_1$. As a result, we can greedily merge clusters in $S_2$, as illustrated in step(3). 

Assume the number of clusters decreases by $D$ after merging, we apply a momentum update~\cite{qian1999momentum} on the number of clusters $N$, as is commonly used in machine learning to smooth the changing of $N$ and avoid sample selection bias. To be specific:
$N_{new}=\alpha (N-D)+(1-\alpha)N$, where $\alpha$ is a hyper-parameter for momentum.

\subsection{Dynamically Determining the Batch Size}
\label{sec.scheduler.batch}
Because of the dynamic grouping operation, the computational graph in deep learning training~\cite{abadi2016tensorflow} varies from sample to sample. As a result, it is impossible to precisely compute a batch's GPU memory usage without indeed feeding it into the model. To overcome this problem, \system learns a batch size prediction function offline; then at the \system training time, given a number of groups $N$, \system uses this function to predict a proper batch size.

When the model architecture and hardware are fixed, the batch size depends on the length of the timeseries $L$ and the average group number among all attention module $\overline{N}$. So \system samples several $(L_i,\overline{N}_i)$ pairs and estimate a proper batch size for each pair. 

More specifically, given a user-defined timeseries maximal length $L_{max}$, we randomly sample integral points $(L_i,N_i)$ from plane $\{1 \leq L \leq L_{max}, 1\leq N \leq L\}$. Then we use a binary search based algorithm to find the maximal batch size $B_i$ that consumes less than $90\%$ available GPU memory, aiming to avoid wasting GPU memory and the risks of out of memory (OOM).

Treating these pairs as ground truth labels, we use function fitting~\cite{guest2012numerical} to learn the batch size predicting function $\mathit{B = f(L,N)}$, where B is a function of two variables $L$ and $N$. 



\noindent\textbf{Learning the Prediction Function.}
We apply \textit{curve fit} from SciPy~\cite{2020SciPy-NMeth} as the function fitting tool to fit the two-variable function $B_i=f(L_i,N_i)$ on plane $\{1 \leq L \leq L_{max}, 1\leq N \leq L\}$.  

We observe that applying one function to the whole plane incurs a huge estimation error. 
So we develop a dynamic-programming (DP) method to divide the plane into several sub-planes and apply a distinct function to each sub-plane respectively. It is {\bf optimal} in minimizing the total estimation error on all sub-planes

With the learned prediction function $f$, we can estimate a proper batch size for any $(L,N)$ during training, even if it is not seen in the sampled $(L_i,N_i)$ pairs. 

\noindent\textbf{The Algorithms and Optimality Proof.} Please refer to Appendix~\ref{appendix.batch} for the pseudo code of the binary search-based algorithm and the description of the DP method for plane-division and the proof for its optimality.


\section{Evaluation}
\label{sec.exp}

Our experimental study focuses on the following questions:

1. \textbf{Effectiveness and efficiency of \system}: How does \system compare with other Transformer-based methods and traditional timeseries representation learning methods in accuracy and efficiency?



2. \textbf{Ablation Study}: How do the key techniques of \system work?

\subsection{Experimental Setup}
\label{sec.exp.setup}
\noindent\textbf{Datasets.}
We evaluate \system on classification and imputation tasks using 5 multi-variate and 3 uni-variate timeseries datasets.

\begin{compactitem}
\item
\textit{\textbf{WISDM}}~\cite{weiss2019smartphone} is a popular multivariate timeseries dataset generated from the accelerometer in the mobile phone.
The subjects performed 18 daily activities (e.g. walking, jogging). The dataset was collected from 51 subjects and the sampling rate is 20 Hz.

\item
\textit{\textbf{HHAR}} dataset~\cite{stisen2015smart} contains sensing data of accelerometer collected
from 9 users performing 5 activities with 12 different smartphones (varying in sampling rate). This increases the complexity of the task and thus can test the model’s robustness.

\item
{\textit{\textbf{RWHAR} RealWorld HAR}} dataset~\cite{sztyler2016body} covers 15 subjects performing 8 locomotion-style activities. Each subject wears the sensors for approximately ten minutes. The sampling rate is 50 Hz.

\item
\textit{\textbf{ECG}} dataset~\cite{liu2018open} consists
of 10,000 EEG recordings for arrhythmia classification. Each recording has an uncertain length ranging from 6 to 60 seconds sampled at 500 Hz. The ECG recordings correspond to 9 types of heart problems such as atrial fibrillation (AF) and premature atrial contraction (PAC), etc. 

\item
{\textit{\textbf{MGH}}}~\cite{DBLP:journals/pvldb/CaoTAJYLGSBSCWM19} is a EEG dataset collected by Mass. General Hospital. Each timeseries corresponds to the EEG data observed from one patient during their stay in ICU for a couple of days. The EEG monitoring produced data with 20 channels. The sampling rate is 200 HZ. So it produces very long timeseries.

\item
{\textit{\textbf{WISDM*/HHAR*/RWHAR*}}}
are three uni-variate datasets derived by picking one channel from \textit{WISDM/HHAR/RWHAR}.

\end{compactitem}

\noindent\textbf{Training/Validation Data Generation.}
We apply a sliding window on the raw timeseries to get training/validation samples. 
The size of the sliding window is set as 200 on small datasets (WISDM, HHAR, RWHAR), 2000 on medium size dataset (ECG), and 10,000 on the large dataset (MGH). Table~\ref{tab.dataset} shows the statics of the generated datasets. They are randomly split into training/validation set in a proportion of 0.9/0.1. 
In ``pretraining + few-label finetuning'' scenario, we use 100 labeled data per class for finetuning. We guarantee that training set does not overlap with the validation set.

\begin{table}[htbp]
\vspace{-3mm}
\centering
\footnotesize
\begin{tabular}{c|c|c|c|c|c}
\toprule
    Dataset & Train. Size & Valid. Size & Length & Channel & Classes \\
    \hline
    WISDM & 28,280 & 3,112 & 200 & 3 & 18\\
    HHAR & 20,484 & 2,296 & 200 & 3 & 5\\
    RWHAR & 27,253 & 3,059 & 200 & 3 & 8\\
    ECG & 31,091 & 3,551 & 2000 & 12 & 9\\
    MGH & 8,550 & 950 & 10000 & 21 & N/A\\
 \bottomrule
\end{tabular}
\caption{The statistics of the datasets}
\label{tab.dataset}
\vspace{-7mm}
\end{table}


\noindent\textbf{Alternative Methods.} 
We compare \system against the SOTA Transformer based timeseries representation learning method {\bf TST}~\cite{DBLP:conf/kdd/ZerveasJPBE21}. 
To evaluate our group attention (referred to as \textbf{Group Attn.}), we develop three baselines by replacing the group attention component in \system with the classic vanilla Self-Attention~\cite{DBLP:conf/nips/VaswaniSPUJGKP17}(referred to as \textbf{Vanilla}) and two SOTA methods that reduce the complexity of self-attention by approximation in NLP, namely, Performer~\cite{choromanski2020rethinking} (referred to as \textbf{Performer}) and Linformer~\cite{wang2020linformer} (referred to as \textbf{Linformer}). 
Similar to our proposed Group Attn., Vanilla, Performer, Linformer all use \system's time-aware convolution operation (Sec.~\ref{sec.rita}) to turn timeseries segments into input feature vectors.

We also compare Group Attn. against \textbf{GRAIL}~\cite{paparrizos2019grail}, which is the SOTA of the non-deep learning methods for timeseries representation learning. GRAIL supports classification tasks by feeding the learned representations into a Support-Vector Machine~\cite{cortes1995support} or K-Nearest Neighbor~\cite{fix1989discriminatory} classifier. 
Note GRAIL only targets {\bf uni-variate} timeseries and cannot support imputation tasks.

\noindent\textbf{Methodology.}
We mainly focus on two downstream tasks:

(1) \textbf{Classification}. 
First, we train Group Attn. and the baselines with full labels from scratch to test the effectiveness of \system framework and the approximation quality of our group attention. 


Second, to measure the effectiveness of self-supervised pretraining, we evaluate the accuracy of training on few labeled timeseries with/without pretraining on large scales of unlabeled timeseries. 
To be specific, we split the training set into a {\it pretraining} set and a {\it finetuning} set, with very few data in the latter (100 labeled samples per class in our experiment). 
We train the model on the cloze pretraining task with a mask rate $p=0.2$. 
Then we train two classification models using the finetuning set, either based on the pretrained version or from scratch. We repeat the experiment 5 times with random data splits and report the median accuracy. 


(2) \textbf{Imputation}. We run the imputation task on the datasets used in classification as well as the large unlabeled MGH dataset, and measure the mean square error and absolute imputation error. To get timeseries with missing values, we randomly mask the values with an expected mask rate of $p=0.2$. The masked values are replaced with a special value. 

Finally, to evaluate Group Attn.'s benefit on {\bf efficiency}, the total time of forward computation, backward propagation, and grouping are measured for all methods in all the experiments. 

To save space, we only report the average training time per epoch here and refer readers to Appendix~\ref{sec.sup.evaltime} for the inference time.

We first compare against the Transformer-based methods on multi-variate datasets (sec.~\ref{sec.exp.effective}, \ref{sec.exp.efficiency}), then compare against the non-deep learning method GRAIL on uni-variate datasets (sec.~\ref{sec.exp.univariate}).

\noindent\textbf{Configuration.} Please refer to Appendix~\ref{appendix.exp} for the experiment configuration and hyper-parameter settings.

\begin{figure}[t]
    \centering
    \includegraphics[width=1.0\columnwidth]{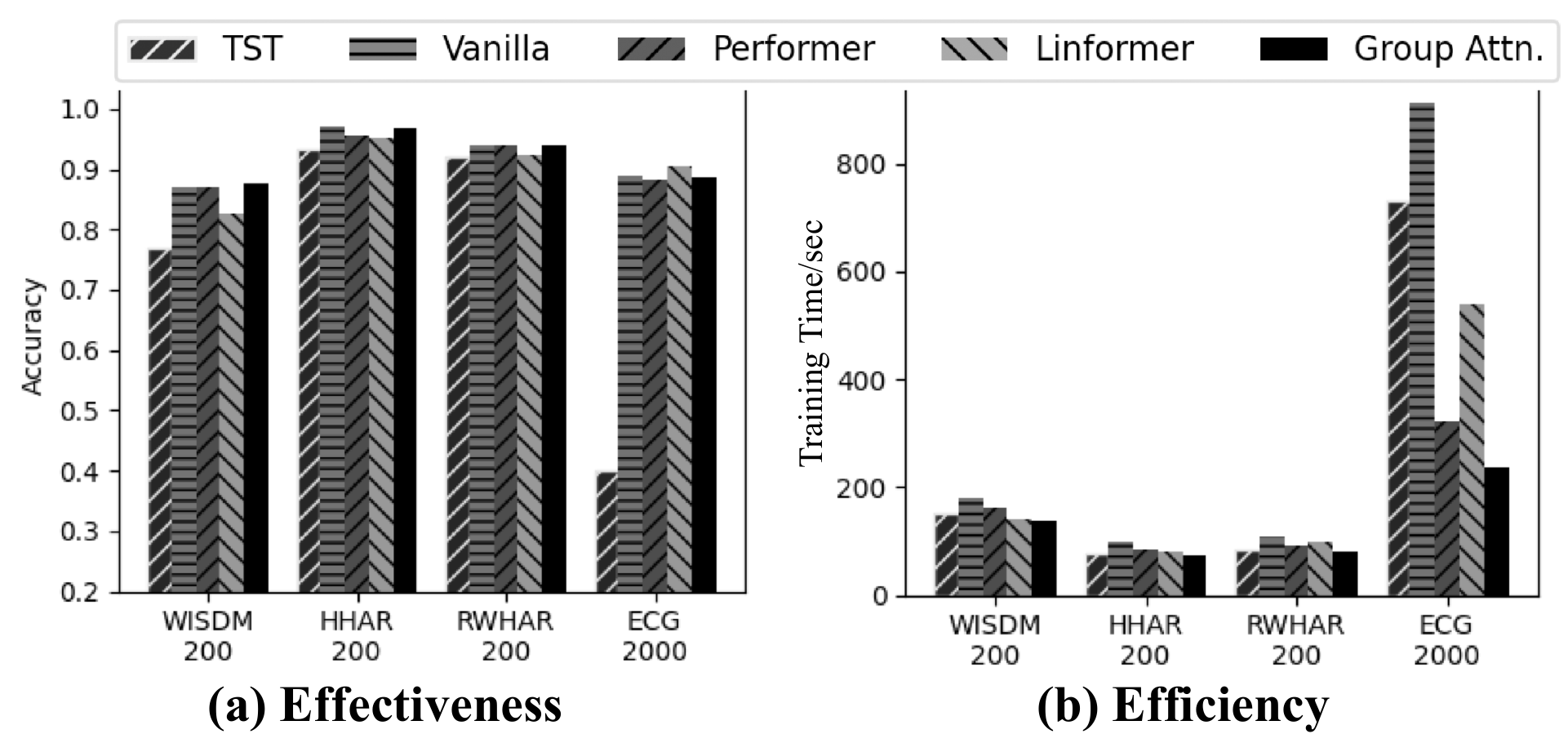}
    \vspace{-6mm}
    \caption{Full-label classification results (multi-variate data).}
    \label{fig.full}
    \vspace{-6mm}
\end{figure}

\begin{table*}[t]
\vspace{-2mm}
\centering
\footnotesize
\begin{tabular}{cc|cc|cc|cc|cc|cc}
\toprule
\multirow{2}{*}{Dataset} &  \multirow{2}{*}{Length} &
\multicolumn{2}{c}{TST~\cite{DBLP:conf/kdd/ZerveasJPBE21}}& \multicolumn{2}{c}{Vanilla} & \multicolumn{2}{c}{Performer} & \multicolumn{2}{c}{Linformer} & \multicolumn{2}{c}{Group Attn.}\\
 \cline{3-12}
 &   &  MSE & Time/s & MSE & Time/s  & MSE & Time/s  & MSE & Time/s  & MSE & Time/s \\
 \hline
 WISDM & 200 & 13.30 & 150.3 & \underline{\textbf{3.240}} & 178.1 & 3.449 & 162.6 & 3.852 & 141.9 & 3.277 & \underline{\textbf{136.7}}  \\
 HHAR & 200 & 1.085 & 78.2 &  \underline{\textbf{0.2968}} &  97.4 & 0.2980 & 82.6 & 0.3198 & 81.1 & 0.2974 & \underline{\textbf{73.3}} \\
 RWHAR & 200 & 0.0882 & 83.9 & \underline{\textbf{0.0478}} & 108.1 & 0.0489 & 89.1 & 0.0572 & 98.4 & \underline{\textbf{0.0478}} & \underline{\textbf{81.3}}\\
  ECG & 2000 & 0.0905 &	696.3 &	0.0037	& 857.9 & 	\underline{\textbf{0.0033}} &	270.2 &	0.0035 &	291.38 &	0.0038 &	\underline{\textbf{164.36}}\\
 MGH  & 10000 & N/A & N/A & N/A & N/A & \underline{\textbf{0.00014}} & 356.2 & 0.00088 & 404.9 & 0.00042 & \underline{\textbf{54.4}} \\
 \bottomrule
\end{tabular}
\caption{Imputation results (multi-variate data). The best results are marked with \underline{bold}.}
\label{tab.imputationMulti}
\vspace{-5mm}
\end{table*}

\begin{table*}[t]
\vspace{-2mm}
\centering
\footnotesize
\begin{tabular}{cc|cc|cc|cc|cc|cc}
\toprule
\multirow{2}{*}{Dataset} &  \multirow{2}{*}{Pretrain Size} &
\multicolumn{2}{c}{TST~\cite{DBLP:conf/kdd/ZerveasJPBE21}}& \multicolumn{2}{c}{Vanilla} & \multicolumn{2}{c}{Performer} & \multicolumn{2}{c}{Linformer} & \multicolumn{2}{c}{Group Attn.}\\
 \cline{3-12}
 &   &  Scratch & Pre. & Scratch & Pre. & Scratch & Pre. & Scratch & Pre. & Scratch & Pre.\\
 \hline
 WISDM & 62,231 & 49.13\% & 50.03\% & 66.16\% & \underline{\textbf{75.89\%}} & 66.09\% & 73.97\% & 50.12\% & 67.44\% & 62.56\% & 75.06\%  \\
 HHAR & 68,294 & 72.56\% & 75.30\% & 75.60\% & 81.35\% & 76.52\% & 80.70\% & 65.94\% & 76.52\% & 76.17\% & \underline{\textbf{82.62\%}}\\
 RWHAR & 63,599 & 69.46\% & 80.41\% & 85.68\% & 91.14\% & 87.54\% & \underline{\textbf{91.33\%}} & 81.03\% & 86.33\% & 86.13\% & 89.63\%\\
 ECG  & 561,358 & 20.98\% & 27.99\% & 42.05\% & 46.16\% & 43.34\% & 45.58\% & 27.19\% & 31.34\% & 42.58\% & \underline{\textbf{46.39\%}} \\
 \bottomrule
\end{tabular}
\caption{Pretrain + few-label finetuning results. The best results are marked with \underline{bold}.}
\label{tab.pretrain}
\vspace{-5mm}
\end{table*}

\subsection{Effectiveness: Transformer-Based Methods}
\label{sec.exp.effective}
We first evaluate the quality of the models trained with full labels from scratch. We then show how the pretraining of \system increases the accuracy of the downstream tasks. 

\subsubsection{full-label training (Multi-variate classification) \nopunct}\ \\
Results shown in Figure~\ref{fig.full}(a) get us the following observations: 

\textbf{(1) \system's advantage over TST.} On all four datasets for the classification tasks,  Group Attn. and the other three baselines that use \system architecture (Vanilla, Performer, and Linformer) outperform TST. In particular, Group Attn. outperforms TST by 49 percentage points on the ECG dataset (88.48\% vs 39.93\%) with long timeseries.
Two deficiencies in TST may cause its poor performance on the long timeseries. Firstly, TST concatenates the output embedding vector of each time stamp, then uses a linear classifier to do classification on the concatenated vector. When the timeseries is long, the linear classifier has so many parameters that it tends to overfit easily. Secondly, TST replaces Layer Normalization in vanilla Transformer with Batch Normalization. 
When the timeseries is long, it can only accommodate a small number of timeseries in each batch, leading to bias in Batch Normalization.

\textbf{(2) Group-attention's advantage over other attention mechanisms.} 
Group Attn. is better than Performer and Linformer on 3 out of 4 datasets for classification. Although Linformer works slightly better than Group Attn. on the ECG dataset (90.37\% vs 88.84\%), its performance is the worst in all other cases compared to any other \system-based methods.
Vanilla computes the attention scores precisely. Thus it is expected to work well. However, Group Attn. outperforms Vanilla on WISDM (87.50\% vs 86.95\%) and is very close to it on other 3 datasets.
This suggests that group attention's approximation quality is good.


\subsubsection{pretraining + few label finetune (Multi-variate classification)\nopunct}\ \\
The results shown in Table~\ref{tab.pretrain} get us the following observation: 

\textbf{(1) Pretraining is effective.} Pretraining always leads to better accuracy than training with a few labels from scratch. In particular, on WISDM data all the methods using \system architecture increase the accuracy by at least 10\%. This is impressive considering we do not have a very large unlabeled pre-training set to use.

\textbf{(2) \system's advantage over TST.} our Group Attn. and other three baselines using \system architecture (Vanilla, Performer, and Linformer) significantly outperform TST on all four classification datasets by 25 percentage points.

\textbf{(3) Group Attention's advantage over other attention mechanisms.} Group Attn. is better than Performer and Linformer on 3 out of 4 datasets. When compared to Vanilla, Group Attn. is better on HHAR and ECG, and comparable on the other two, further confirming its high quality on approximation. 
Further, we notice that Linformer struggles in this setting: in average its accuracy is worse than Vanilla by 8.22\% and Group Attn. by 8.01\%. 
This is because the low-rank projection operation introduces extra model parameters, making Linformer more easily overfit, while overfitting is especially harmful when there are only a few labeled training samples.

\subsubsection{full-dataset training (Multi-variate imputation)\nopunct}\ \\

Similar to classification tasks, the results of {\bf imputation tasks} (Table.\ref{tab.imputationMulti}) show that Group Attn. consistently outperforms the baselines in training time while achieving comparable/better MSE. Again, on the large dataset MGH (length = 10,000), TST and Vanilla fail due to out of memory (OOM) errors.
Methods using \system framework (Group Attn., Performer, Linformer) all achieve very low MSE (are highly accurate). Among them Linformer is the worst.

\subsection{Efficiency: Transformer-based Methods}
\label{sec.exp.efficiency}
We measure the efficiency by the average training time per epoch including the cost of the forward computation + backward propagation and the grouping overhead. 
We first show the results on all the 5 datasets in Sec.~\ref{sec.exp.efficiency.all}. We then vary the length of the timeseries on the MGH dataset to show group attention's scalability on long timeseries in Sec.~\ref{sec.exp.efficiency.length}. 

\subsubsection{Training Time: All Multi-variate Datasets\nopunct}\ \\
\label{sec.exp.efficiency.all}
The results in Fig.~\ref{fig.full}(b) and Table~\ref{tab.imputationMulti} lead to the below observations: 

\textbf{(1) Vanilla Self-Attention is not scalable.} In average, it takes 2-3 minutes to train one epoch when the length of the timeseries is only 200 (WISDM, HHAR, RWHAR), takes over 15 minutes when the length increases to 2,000 (ECG), and fails on the long MGH data when the length reaches 10,000 due to out of GPU memory.  

\textbf{(2) Group Attn.'s advantage over all other attention mechanisms.} 
As we have shown in Sec.~\ref{sec.exp.effective}, Group Attn. is more accurate than Performer and Linformer in classification and imputation tasks, while Group Attn. is always faster than Performer, Linformer, and all other baselines on all 5 multi-variate datasets, thus a {\bf win-win}.

\textbf{(3) The longer the timeseries, the larger the speedup.} 
On the medium sized ECG dataset with a length of 2,000, Group Attn. has a speedup of 3.86/1.36/2.27 compared to Vanilla/Performer/Linformer. When the length increases to 10,000, the speedup on the MGH dataset increases to 6.59/7.48 compared to Performer/Linformer (Vanilla and TST failed in this case) on imputation task (Table.~\ref{tab.imputationMulti}). 
However, even on the short WISDM, HHAR, RWHAR datasets, Group Attn. still consistently outperforms other methods, confirming that it does not introduce much overhead.  
This is because when the length of the timeseries gets longer, Group Attn. gets more opportunities to find windows with similar properties.

\begin{figure}[t]
    \centering
    \includegraphics[width=1.0\columnwidth]{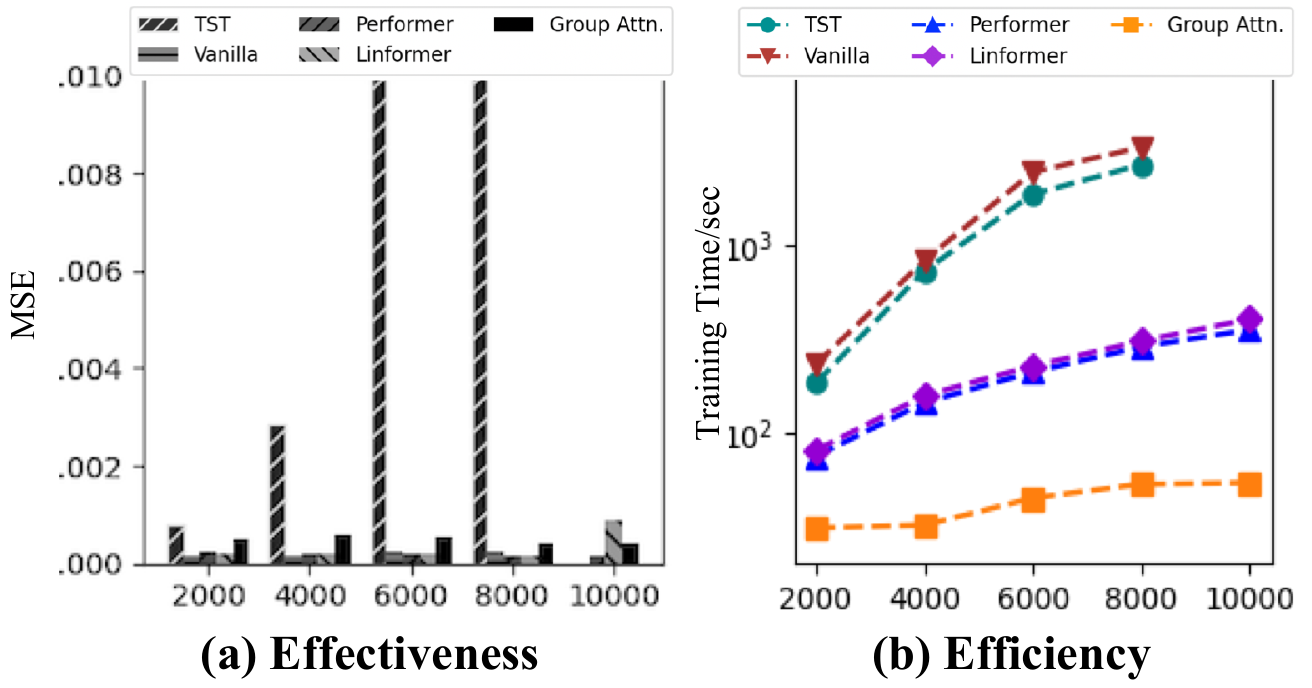}
    \vspace{-7mm}
    \caption{Varying the lengths of timeseries.}
    \label{fig.length}
    \vspace{-2mm}
\end{figure}


\subsubsection{Training time: Varying the Length\nopunct}\ \\
\label{sec.exp.efficiency.length}
In this experiment, we truncate the original MGH timseries into sequences with the lengths at 2000/4000/6000/8000/10000, and compare Group Attn. against Vanilla and other attention mechanisms. Vanilla cannot handle sequences longer than 8000.

The results in Fig.~\ref{fig.length} again show that \textit{the longer the timeseries, the larger the speed up}. With comparable MSE, Group Attn. outperforms Vanilla by 63X.
Moreover, as the length increases from 2000 to 10000, the training time of Group Attn. only increases from 31.2 seconds to 54.4 seconds per epoch.
The reason is that as the timeseires becomes longer, there are more grouping opportunities because of the similarity of the timeseries segments. 

\subsection{Comparison to Non-deep Learning Methods}
\label{sec.exp.univariate}

\begin{figure}[t]
\vspace{-2mm}
    \centering
    \includegraphics[width=0.8\columnwidth]{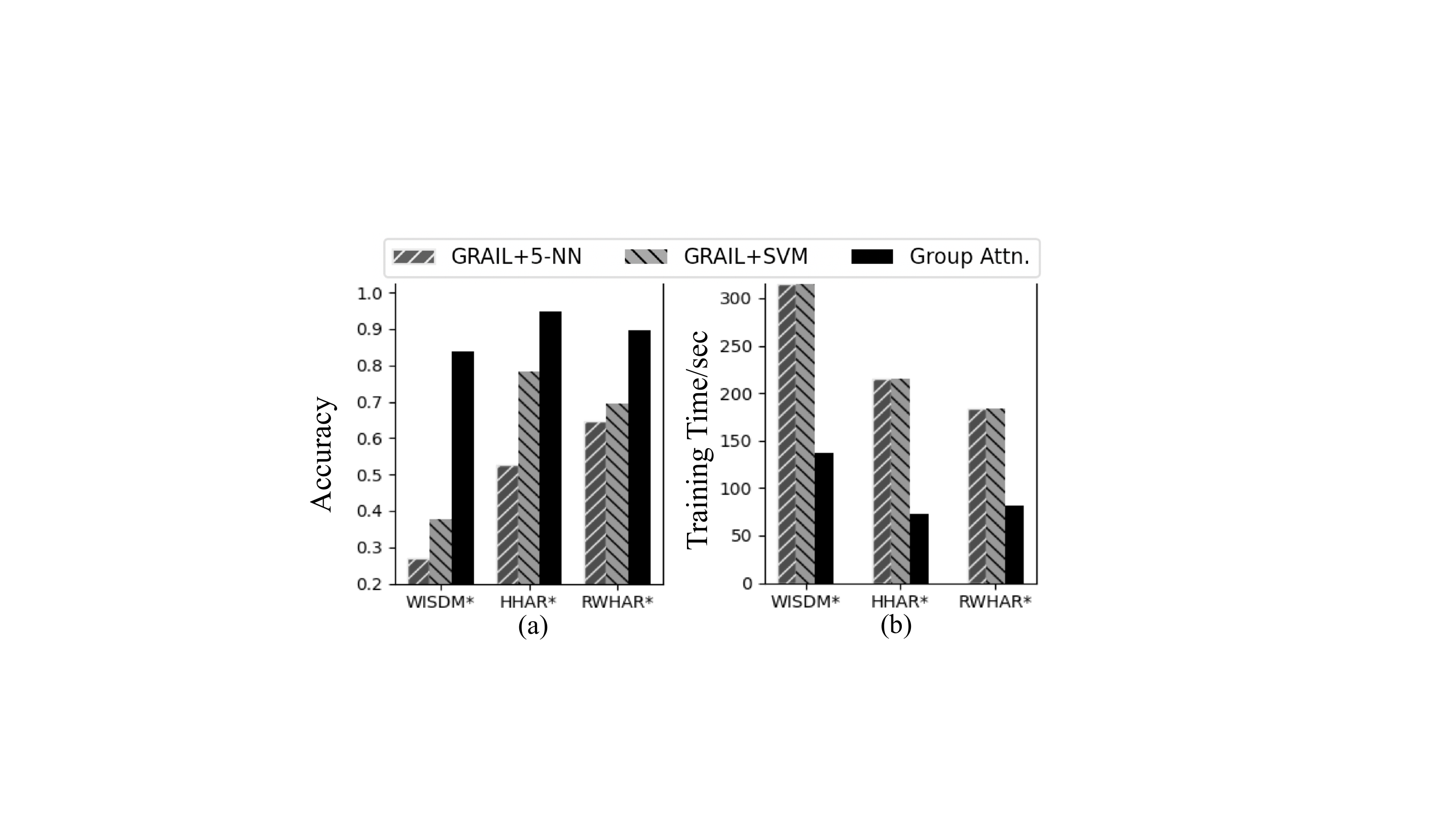}
    \vspace{-4mm}
    \caption{Comparison to non-deep learning method (uni-variate data).}
    \label{fig.full_uni}
    \vspace{-5mm}
\end{figure}

We compare against GRAIL, the SOTA of non-deep learning timeseries representation learning. We use the three uni-variate datasets, because GRAIL only targets uni-variate timeseries.

Results in Fig.~\ref{fig.full_uni} show that on all 3 datasets \system significantly outperforms GRAIL in accuracy by 45, 16, and 21 percentage points because of the expressive power of Transformer.
Moreover, thanks to the GPU-friendly design of \system, it is at least 2$\times$ faster than GRAIL in training time.

\subsection{Ablation Study}
\label{sec.exp.ablation}

\subsubsection{Adaptive Scheduler\nopunct}\ \\
To evaluate the effectiveness of \system's adaptive scheduler (Sec.~\ref{sec.scheduler}), we compare it against a baseline using a fixed group number $N$. We vary $N$ and the error bound threshold $\epsilon$ used by \system.   

From the results in Table~\ref{tab.dynamic} we get the following observations: 

\textbf{(1) Adaptive Scheduler is better than fixed $N$.} Training with Adaptive Scheduler already achieves better or comparable performance compared to the best performing $N$. More specifically, on the MGH dataset, dynamic scheduler always achieves better accuracy and is much faster compared to fixed $N$.
On the ECG dataset, although fixed $N$ is slightly better than adaptive scheduler in accuracy when setting the N as 512, it runs much slower than adaptive scheduler. 
Of course, finding the best $N$ that balances the accuracy and running time requires careful tuning.  
 
\textbf{(2) Adaptive Scheduler is tuning free.} It is robust on both accuracy and running time when $\epsilon$ varies, while the results of fixed $N$ vary significantly when the value of $N$ changes.
Therefore, Adaptive Scheduler frees the users from tuning the $\epsilon$ threshold, while it is hard to find an appropriate $N$ for a given dataset.

\begin{table}[htbp]
\vspace{-2mm}
\centering
\footnotesize
\begin{tabular}{cc|c|c|cc}
\toprule
Dataset & Task & Scheduler & Parameter & Metric & Time\\
 \hline
 \multirow{6}{*}{ECG} & \multirow{6}{*}{Class.} & \multirow{3}{*}{Dynamic} & 1.5 & 88.34\% & 292.5  \\
  & &  & 2 & 88.48\% & 236.8 \\
  &   &  & 3 & 87.83\% & 216.8 \\
   \cline{3-6}
   &   & \multirow{5}{*}{Fixed}  & 64 & 87.50\% & 255.2\\
   &   &   & 128 & 88.96\% & 297.2\\
   &   &  & 256 & 88.82\% & 414.1 \\
   &   &  & 512 & 90.03\% & 662.6 \\
   &   &  & 1024 & 88.65\% & 873.7\\
   \hline \hline
   
 \multirow{6}{*}{MGH} & \multirow{6}{*}{Imput.} & \multirow{3}{*}{Dynamic} & 1.5 & 0.00041 &  60.7  \\
  & &  & 2 & 0.00040  & 57.9 \\
  &   &  & 3 & 0.00042  & 54.4   \\
   \cline{3-6}
   &   & \multirow{4}{*}{Fixed}  & 128 & 0.00054 & 128.6 \\
   &   &  & 256  & 0.00053 & 190.2\\
   &   &  & 512  & 0.00049 & 240.8\\
   &   &  & 1024  & 0.00046 & 323.3 \\
 \bottomrule
\end{tabular}
\caption{Adaptive Scheduling VS Fixed N.}
\label{tab.dynamic}
\vspace{-3mm}
\end{table}

\begin{table}[htbp]
\vspace{-4mm}
\centering
\footnotesize
\begin{tabular}{c|c}
\toprule
 Pretrain Data size & Few-label Accuracy \\
 \hline
   N/A & 62.56\%   \\
   \hline
   12,446 & 72.94\%\\
   24,892 &  72.78\%\\
   37,338 & 74.10\%\\
   49,784 & 74.22\%\\
   62,231 & 75.06\% \\
 \bottomrule
\end{tabular}
\caption{\system Pretraining: increasing sizes of pretrain set.}
\label{tab.pretrainsize}
\vspace{-8mm}
\end{table}

\vspace{-1mm}
\subsubsection{The Sizes of the Pretraining Data\nopunct}\ \\
Next, we evaluate how the number of unlabeled data influences the effectiveness of pretraining. 
To get empirical results, we pretrain \system on WISDM dataset with 20\%/40\%/60\%/80\% of the pretraining data and finetune each pretrained model with 100 labels per class. 
The results in Table~\ref{tab.pretrainsize} show that: 
\textbf{(1) The more pretraining data, the larger the improvement.} The accuracy increases with the sizes of the pretraining data; \textbf{(2) Marginal utility diminishing.} The first 20\% pretraining data gives a 10.38\% improvement in accuracy (72.94\% vs 62.56\%), while the remaining 80\% pretraining data only gives an additional improvement of 2.12\% (75.06\% vs 72.94\%). 

\section{Related work}
\label{sec.related}

\subsection{Timeseries Analytics}
There is a great deal of prior work on timeseries analytics methods. This work can be divided into three categories: (1) non-deep learning methods; (2) CNN/RNN-based deep learning methods; and (3) Transformer-based deep learning methods.

\noindent\textbf{Traditional Methods.} 
These methods, such as TS-CHIEF~\cite{Shifaz2020TSCHIEFAS}, HIVE-COTE~\cite{10.1145/3182382},  ROCKET~\cite{DBLP:journals/datamine/DempsterPW20} have achieved notable performance on public datasets. Despite that, traditional methods suffer from one or more issues: they (1) rely on expert knowledge for feature extraction;
(2) incur heavy computation cost and are inappropriate for GPU devices; (3) support only uni-variate timeseries; (4) perform classification solely. 
Some work~\cite{DBLP:conf/kdd/ZerveasJPBE21} shows that the transformed-based methods outperform these traditional methods especially on multi-variate timeseries. 


In particular, as the SOTA of timeseries {\bf representation learning}, GRAIL~\cite{paparrizos2019grail} extracts landmarks from data and computes the representations with the combination of the landmarks. However, GRAIL only supports uni-variate timeseries. Our experiments (Sec.~\ref{sec.exp.univariate}) show that \system significantly outperforms GRAIL in both effectiveness and efficiency on uni-variate timeseries.



\noindent\textbf{CNN/RNN-based Deep Learning Methods.}
CNN-based methods, such as InceptionTime~\cite{ismail2020inceptiontime} and Resnet~\cite{he2016deep}, are good at classification tasks, but can not handle generative tasks such as forecasting because of the inductive bias of convolution networks. 
RNN-based methods, such as Brit~\cite{cao2018brits} and deepAR~\cite{salinas2020deepar}, are capable for classification, regression and generation. However, the recurrent structure brings a lot of problems: (1) limiting the model's ability in capturing long-range correlation; (2) notoriously difficult to train~\cite{pascanu2013difficulty} because of gradient vanishing and exploding problem. As a result, such methods can hardly scale to very long timeseries. 

\noindent\textbf{Transformer-based Deep Learning Methods.}
Given that Transformer is the best choice for backbone in almost all sequence modeling tasks, some effort has been made to apply Transformer to timeseries analytics.
Targeting forecasting of uni-variate timeseries, LogTrans~\cite{li2019enhancing} introduced a log sparsity assumption to attention computation.
Informer~\cite{zhou2021informer} pushes LogTrans a step further and scales forecasting to multi-variate timeseries. Autoformer~\cite{wu2021autoformer} performs forecasting by decomposing timeseries into two parts, i.e. the trend part and the seasonal part.

For imputation tasks, CDSA~\cite{ma2019cdsa} outperforms statistical methods and the SOTA of RNN-based method Brit~\cite{cao2018brits} on 3 public and 2 competition datasets. 
For timeseries classification, AutoTransformer~\cite{ren2022autotransformer} performs architecture search to adapt to the tasks in different domains. 
For timeseries anomaly detection, Anomaly Transformer~\cite{xu2021anomaly} outperforms many widely-used methods such as OmniAnomaly~\cite{su2019robust}, assuming the attention score maps show Gaussian distribution.

All of these works are designed for specific tasks, rather than functioning as a {\bf representation learning} framework to serve different downstream tasks. To fill this gap, some researchers proposed a Transformer-based architecture, called TST~\cite{DBLP:conf/kdd/ZerveasJPBE21}. Like \system, TST supports regression, classification, and unsupervised learning through the ``cloze test'' pretraining task on timeseries. 
However, TST directly uses the classical Vanilla self-attention, thus not scalable to long timeseries as shown in our experiments (Sec.~\ref{sec.exp.efficiency.length}).

\subsection{Efficient Transformers}
The need of improving the scalability of Transformers has led to more efficient variations of Transformers, especially for accommodating long text data in NLP~\cite{tay2020efficient}. 

Introducing fixed/random patterns to self-attention mechanism is an intuitive idea. Sparse Transformer~\cite{child2019generating} and Longformer~\cite{beltagy2020longformer} only compute attention at fixed intervals. ETC~\cite{ainslie2020etc} and BigBird~\cite{zaheer2020big} use global-local attention: the attention computation is limited within a fixed radius, while some auxiliary tokens are added to attend/get attended globally. 
The deficiencies of fixed attention patterns are obvious: it heavily depends on users to give an optimal setting.
  

To decrease the reliance on human labor, some works seek to introduce learnable/adaptive attention patterns instead of fixed patterns. Reformer~\cite{kitaev2020reformer} proposed  only computing the dominant attention terms based on their observation of sparsity in attention matrix from language/image data. Such sparsity is intuitive in language data, in which a word's attention mainly focuses on the nearby sentences. However, attention in timeseries data shows strong seasonal patterns rather than sparse patterns, mainly as result of  the periodicity of timeseries data. Therefore, such works do not work well for timeseries.


Apart from introducing attention patterns, some works seek to solve this problem with 
applied mathematics techniques. Linformer~\cite{wang2020linformer} performs a projection to decrease the size of query, key and value matrices before attention computation, because the attention matrix tends to be low-ranked. 
Performer~\cite{choromanski2020rethinking} uses linear functions to approximate the kernel function \textit{softmax}, making attention computation commutative. When the sequence length is far greater than the dimension of embedding vectors, Performer benefits from changing the order of matrix multiplication.
Linformer and Performer do not depend on the unique properties of language data, thus potentially fitting timeseries better than other techniques, which is why we compared against them in our experiments. 
However as shown in Sec.~\ref{sec.exp}, our group attention significantly outperforms them in both accuracy and efficiency (training time), because group attention fully leverages the periodicity of timeseries.

\section{Conclusion}
\label{sec.conclusion}
In this work, we presented \system, an automatic, self-supervised, and scalable timeseries analytics tool. \system effectively adapts Transformer, popular in NLP, into timeseries analytics. 
As the key component of \system, group attention eliminates the performance bottleneck of the classical self-attention mechanisms, thus successfully scaling \system to highly complex, long timeseries data.
Our experiments confirm that \system significantly speeds up the state-of-the-art by 63X with a better accuracy. 



     \bibliographystyle{ACM-Reference-Format}
        \bibliography{reference.bib}
    

\appendix

\begin{sloppypar}
\section{Appendix: Supplementary Material}
\subsection{Experiment Configuration and Hyper-parameter Settings}
\label{appendix.exp}
\noindent\textbf{Configuration.} All models were trained on an NVIDIA Tesla V100 16GB GPU. All the methods are optimized with AdamW~\cite{loshchilov2017decoupled} of which the starting learning rate and weight decay parameter are both $1e^{-4}$. In full-label training scenario, we train the models for 100 epochs. In ``pretraining + few-label finetuning scenario'', as the pretrained models require fewer epochs to converge~\cite{DBLP:conf/kdd/ZerveasJPBE21}, we train the model for 50 epochs. For a fair comparison, the baselines use a maximal batch size within GPU's capacity during training.

As for model hyper-parameter setting, \system and the baselines use a Transformer structure balancing \textit{\textbf{Vanilla}} 's accuracy and efficiency: 8-layer stack of 2-head attention with hidden vectors in dimension of 64. Convolution kernel size is set to 5 by default.
We set the error bound threshold ($\epsilon$, Sec.~\ref{sec.scheduler.group}) of Group Attention to 2, as it balances the accuracy and the efficiency in general on all datasets.
Because Linformer requires the users to set the sizes of projection matrix, in different settings we choose an accuracy-efficiency balancing one among \{64,128,256,512\}.

\subsection{Efficient Computation of Group Attention}
\label{appendix.groupAttention}
\begin{algorithm}
    \caption{Efficient Computation of Group Attention}
    \label{algo.grpattn}
    \small
    \begin{algorithmic}[1] 
    \Require $Q,V,R,COUNT,BELONG$
    \Ensure $Q,V \in \mathbb{R}^{n*d}$,$ R \in \mathbb{R}^{N*d}$,$COUNT \in \mathbb{N}^{N}$,$BELONG \in \mathbb{N}^{n}$
            \Function {group\_attention}{$Q,V,R$}
                \For{$i = 0 \to N-1$}
                    \State $\widetilde{v}_i \gets \sum_{j=0}^{n-1}(BELONG_j==i) v_j$
                \EndFor
                
                \State $\widetilde{P} \gets QR^T$
                
                \For{$i = 0 \to n-1$}
                    \For{$j = 0 \to N-1$}
                    \State $w_{i,j} \gets exp(\widetilde{P}_{i,j})COUNT_j$
                    \EndFor
                \EndFor
                
                \For{$i = 0 \to n-1$}
                    \State $s_i \gets \sum_{j=0}^{N-1} w_{i,j}$ 
                \EndFor
                
                \For{$i = 0 \to n-1$}
                    \State $o_{i} \gets \sum_{j=0}^{N-1}\frac{exp(\widetilde{P}_{i,j})}{s_i}\widetilde{v}_j$
                \EndFor
                    
                \State \Return{$O$}
            \EndFunction
            
        \end{algorithmic}
\end{algorithm}

In Alg.~\ref{algo.grpattn}, we denote $COUNT_i$ to be the size of the $i^{th}$ group, $N$ to be the number of groups, $\mathbf{r}_i$ to be the representative key of the $i^{th}$ group and $\mathbf{R}$ to be the matrix consisting of all $\mathbf{r}_i$, $BELONG_i$ to be the group that $\mathbf{k}_i$ belongs to. $Q,V$ are the packing matrices of query vectors and value vectors as described in Sec.\ref{sec.preliminary}. Alg.~\ref{algo.grpattn} outputs the packing matrix $O$ for new feature emebddings $\{o_1,...,o_n\}$, where $o_i$ corresponds to the feature embedding of $win_i$. 
Lines 2-3 implement the embedding aggregation operation, while Lines 8-11 implement the group softmax function.

\subsection{The Algorithms and Optimality Proof for Dynamically Determing Batch Size}
\label{appendix.batch}

\begin{algorithm}[h]
    \caption{Binary Search for Batch Size}
    \label{algo.binsearch}
    \begin{algorithmic}[1] 
    \Require $L,N$
    \Ensure $1 \leq L \leq L_{max}, 1\leq N \leq L$
            \Function {binary\_search}{$L,N$}
                \State $L \gets 1$
                \State $R \gets MaxBatchSize$
                \State $data \gets RandomTimeSeries\ in\ length\  L$
                \State $B_{temporal}$
                \While{$L \leq R$}
                     
                    \State $Input \gets data \times B_{temporal}$
                    \State $ModelForward(Input)$
                    \State $ModelBackward$
                    \State $u \gets  \frac{PeakMemoryUsage}{TotalMemory}$
                    
                    \If{$0.9 > u$} 
                    \State {$L \gets B_{temporal}+1$}
                    \State {$B \gets B_{temporal}$}
                    \Else 
                    \State{$R \gets B_{temporal}-1$}
                    \EndIf
                    \State $B_{temporal}\gets \frac{\lfloor L+R \rfloor}{2}$
                \EndWhile
                
                \State \Return{$B$}
            \EndFunction
            
        \end{algorithmic}
\end{algorithm}

\begin{algorithm}[h]
    \caption{Dynamic Programming for Plane Division}
    \label{algo.dpdiv}
    \footnotesize
    \begin{algorithmic}[1] 
    \Require $L_i,N_i,B_i,L_{max}$
    \Ensure $1 \leq L_i \leq L_{max}, 1\leq N_i \leq L_i$
            \Function{cost}{S}
                \If{$|S|<M$}
                    \Return{$+\infty$}
                \EndIf
                \State{$L,N,B \gets points\ in\ S$}
                \State{$f \gets function\ fitting(B|L,N)$}

\Return{$E(B,L,N|f)$}
            \EndFunction
                
            \Function {dynamic\_programming}{$L_i,N_i,L_{max}$}
                \For {$l_1=1 \to L_{max}$}
                    \For {$l_2=1 \to l_1$}
                        \For{$n=1 \to l_1$}
                            \State{$S \gets points\ set\ in\ \{l_2 \leq L \leq l_1,N \leq n\}$}
                            \State{$g(n) \gets COST(S)$ }
                            \For{$i=1 \to n$}
                                \State {$S \gets  points\ set\ in\ \{l_2 \leq L \leq l_1, i\leq N \leq n\}$}
                                \State{$g(n) \gets min(g(n),g(i)+COST(S))$}
                            \EndFor
                        \EndFor
                    \State{$f_{l_2,l_1} \gets g(l_1)$}
                    \EndFor
                \EndFor
            
                \State{}
            
                \For {$l=1 \to L_{max}$}
                    \State {$dp(l) \gets f(1,l)$}
                    \For {$i=1 \to l$}
                        \State {$dp(l) \gets min(dp(l),dp(i)+f(i,l))$}
                    \EndFor
                \EndFor
                \Return{$dp(L_{max})$}
            \EndFunction
            
        \end{algorithmic}
\end{algorithm}

We describe Alg.~\ref{algo.dpdiv} and intuitively show its optimality. 
We assume that Scipy~\cite{2020SciPy-NMeth} learns an optimal function in Line 4 so that function COST gives the optimal estimation error when fitting the points in set $S$. 
When fitting very few points, we assign an infinite cost to prevent a biased fitting function (Line 2).  
$g(n)$ denotes the minimal estimation error for points in sub-plane $\{l_2 \leq L \leq l_1, N \leq n\}$. In Lines 11-13, we enumerate all possible ways of cutting $\{l_2 \leq L \leq l_1, N \leq n\}$ horizontally into two sub-plane $\{l_2 \leq L \leq l_1, N \leq i\}$ and $\{l_2 \leq L \leq l_1, i \leq N \leq n\}$ by iterating $i$ from 1 to n. 
Choosing the cutting strategy that minimizes estimation error gets us a $g(l_1)$ with minimal estimation error for sub-plane $\{l_2 \leq L \leq l_1, N \leq l_1\}$, which is recorded as $f_{l_1,l_2}$ in Line 14. 
$dp(l)$ denotes the minimal estimation error for sub-plane $\{L \leq l\}$. 
We enumerate all the possible ways of cutting $\{ L \leq l\}$ vertically into two sub-plane $\{  L \leq i\}$ and $\{i \leq L \leq l\}$ by iterating $i$ from 1 to $l$ (Line 17-19). Finally, we have the minimal estimation error for the whole plane as $dp(L_{max})$. 
Based on the above discussion, this algorithm guarantees to not miss any better solution, hence optimal.



\subsection{The Correctness of Group Attention}
\label{appendix.proof.groupAttention}
\begin{lemma}
\label{lm.grpattnalgo}
\sloppypar
Assuming the windows belonging to the same group $G_i$ have the same key vector, i.e. $k_j=r_i (win_j \in G_i)$, then the feature embedding $O$ produced by the original self-attention mechanism is identical to the output of our group attention mechanism implemented in Algorithm~\ref{algo.grpattn}.

\end{lemma}

\begin{proof}
Denote $\widetilde{k_j}$ to be the representative vectors of $k_j$, i.e. $\widetilde{k_j}=r_i=k_j (win_j \in G_i)$. Algorithm~\ref{algo.grpattn} gives that
\begin{equation}
\label{eq.grpres}
\small
\begin{aligned}
    \widetilde{v}_i&=\sum_{j=0}^{n-1}(BELONG_j==i)\mathbf{v}_j, \ \widetilde{P}_{i,j}=\mathbf{q}_i \cdot \mathbf{r}_j\\
    s_i&=\sum_{j=0}^{N-1}exp(\widetilde{P}_{i,j})COUNT_j, \ \widetilde{o}_i=\sum_{j=0}^{N-1}\frac{\widetilde{P}_{i,j}}{s_i}\widetilde{v}_j
\end{aligned}
\end{equation}

By the canonical self-attention mechanism introduced in Sec.~\ref{sec.preliminary}, we get:
\begin{equation}
\small
\label{eq.grpattn1}
P_{i,j}=\mathbf{q}_i \cdot \mathbf{k_j},\ A_{i,j}=\frac{exp(P_{i,j})}{\sum_{k=0}^{n-1}exp(P_{i,k})}, \ \mathbf{o}_i=\sum_{j=0}^{n-1}A_{i,j}\mathbf{v}_j
\end{equation}



With \ref{eq.grpres} and \ref{eq.grpattn1}, we have
\begin{equation}
\label{eq.sumexpahat}
\small
\begin{aligned}
    \sum_{j=0}^{n-1}exp(P_{i,j})&=\sum_{j=0}^{n-1}exp(\mathbf{q}_i \cdot \mathbf{k}_j)\\
    &=\sum_{j=0}^{N-1} \sum_{x=0}^{n-1}(BELONG_x==j)exp(\mathbf{q}_i \cdot \mathbf{k}_x)\\
    &=\sum_{j=0}^{N-1} exp(\mathbf{q}_i \cdot \mathbf{r}_j) \sum_{x=0}^{n-1} (BELONG_x==j) \\
    &=\sum_{j=0}^{N-1} exp(\mathbf{q}_i \cdot \mathbf{r}_j) COUNT_j
    \\
    &=\sum_{j=0}^{N-1} exp(\widetilde{P}_{i,j}) COUNT_j\\
    &=s_i\\
\end{aligned}
\end{equation}

Further,
\begin{equation}
\label{eq.output}
\small
\begin{aligned}
\mathbf{o}_i&=\sum_{j=0}^{n-1} A_{i,j}\mathbf{v_j} \\
&=\sum_{j=0}^{N-1}\sum_{x=0}^{n-1} (BELONG_x==j)A_{i,x}\mathbf{v}_x \\
&=\sum_{j=0}^{N-1}\sum_{x=0}^{n-1}(BELONG_x==j)\frac{exp(P_{i,x})}{\sum_{k=0}^{n-1}exp(P_{i,k})}\mathbf{v}_x\\
&=\sum_{j=0}^{N-1}\sum_{x=0}^{n-1}(BELONG_x==j)\frac{exp(\mathbf{q}_i \cdot \mathbf{k}_x)}{\sum_{k=0}^{n-1}exp(P_{i,k})}\mathbf{v}_x\\
&=\sum_{j=0}^{N-1}\sum_{x=0}^{n-1}(BELONG_x==j)\frac{exp(\mathbf{q}_i \cdot \mathbf{r_j})}{\sum_{k=0}^{n-1}exp(P_{i,k})}\mathbf{v}_x\\
&=\sum_{j=0}^{N-1} \frac{exp(\mathbf{q}_i \cdot \mathbf{r_j})}{\sum_{k=0}^{n-1}exp(P_{i,k})} \sum_{x=0}^{n-1}(BELONG_x==j)\mathbf{v}_x\\
&=\sum_{j=0}^{N-1} \frac{exp(\mathbf{q}_i \cdot \mathbf{r_j})}{\sum_{k=0}^{n-1}exp(P_{i,k})} \widetilde{v_j}\\
\end{aligned}
\end{equation}

Combining \eqref{eq.grpres}, \eqref{eq.sumexpahat} \eqref{eq.output}, we have
$\mathit{\mathbf{o}_i=\sum_{j=0}^{N-1}\frac{\widetilde{P}_{i,j}}{s_i}\widetilde{v}_j=\widetilde{o}_i}$.

This concludes that the output of our group attention is identical to vanilla self-attention's.
\end{proof}

\subsection{The Proof of Error Bound (Lemma 1)}
\label{appendix.proof.errorBound}
\begin{proof}
\renewcommand{\qedsymbol}{}
We have
\begin{equation}
\begin{aligned}
    \frac{exp(\overline{P}_{i,j})}{exp(P_{i,j})}&=\frac{exp({\mathbf{q}}_i \cdot \widetilde{\mathbf{k}}_j)}{exp(\mathbf{q}_i \cdot \mathbf{k}_j)} = exp({\mathbf{q}}_i \cdot (\widetilde{ \mathbf{k}}_j-\mathbf{k}_j))\\
   &=exp(||\mathbf{q}_i|| \cdot ||\widetilde{\mathbf{k}}_j-\mathbf{k}_j||\cdot cos(\mathbf{q}_i,\widetilde{\mathbf{k}}_j-\mathbf{k}_j))
\end{aligned}
\end{equation}

So
\begin{equation}
\label{equ.abound}
exp(-dR) \leq \frac{exp(\overline{P}_{i,j})}{exp(P_{i,j})} \leq exp(dR)
\end{equation}

Then we have:
\begin{equation}
\label{equ.Aequ}
\begin{aligned}
\frac{\overline{A}_{i,j}}{A_{i,j}}&=\frac{exp(\overline{P}_{i,j})}{\sum_{k=1}^{n} exp(\overline{P}_{i,k})} / \frac{exp({P}_{i,j})}{\sum_{k=1}^{n} exp({P}_{i,k})}\\
&=\frac{exp(\overline{P}_{i,j})}{exp({P}_{i,j})} \frac{\sum_{k=1}^{n} exp({P}_{i,k})}{\sum_{k=1}^{n} exp(\overline{P}_{i,k})}
\end{aligned}
\end{equation}

Combining (\ref{equ.abound}) (\ref{equ.Aequ}), the error is bounded by
\begin{equation}
\label{equ.Abound}
exp(-2dR) \leq \frac{\overline{A}_{i,j}}{A_{i,j}} \leq exp(2dR)
\end{equation}

Thus, if $\mathit{d \leq \frac{\ln(\epsilon)}{2R}}$, $\mathit{\frac{1}{\epsilon} \leq \frac{\overline{A}_{i,j}}{A_{i,j}} \leq \epsilon}$. This proves Lemma~\ref{lm.grperrorbound}.
\end{proof}

\subsection{The Proof of Merge Operation (Lemma 2)}
\label{appendix.proof.merge}
\begin{proof}
\renewcommand{\qedsymbol}{}
Denote the cluster size of $cluster_k$ to be $n_k$.After mergeing, the new center will be: $$c'= \frac{\sum_{i=1}^m n_{k_i}c_{k_i}}{\sum_{i=1}^m n_{k_i}}$$
For $\forall i \in [1,m],\forall x \in cluster_{k_i}$, it holds that:
\begin{equation}
\small
\label{eq.tri}
\begin{aligned}
        |x-c'| &\leq |x-c_{k_i}|+|c_{k_i}-c'|\  (Triangle\  inequality)\\
        &=|x-c_{k_i}|+| \frac{\sum_{j=1}^m n_{k_j}}{\sum_{j=1}^m n_{k_j}}c_{k_i}-
        \frac{\sum_{j=1}^m n_{k_j}c_{k_j}}{\sum_{j=1}^m n_{k_j}}|\\
        &=|x-c_{k_i}|+| 
         \frac{\sum_{j=1}^m n_{k_j}(c_{k_i}-c_{k_j})}{\sum_{j=1}^m n_{k_j}}|\\
        &=|x-c_{k_i}|+\frac{|\sum_{j=1}^m n_{k_j}(c_{k_i}-c_{k_j})|}{\sum_{j=1}^m n_{k_j}}\\
        &\leq |x-c_{k_i}|+\frac{\sum_{j=1}^m n_{k_j}|c_{k_i}-c_{k_j}|}{\sum_{j=1}^m n_{k_j}}\\
        &= \frac{\sum_{j=1}^m n_{k_j}(|c_{k_i}-c_{k_j}|+|x-c_{k_i}|)}{\sum_{j=1}^m n_{k_j}}\\
        &\leq \frac{\sum_{j=1}^m n_{k_j}d}{\sum_{j=1}^m n_{k_j}} =d
\end{aligned}
\end{equation}
\end{proof}

\subsection{Downstream Tasks}
\label{appendix.downstream}
\system supports a variety of downstream tasks. In this section, we show that with minimal modification \system can effectively support classification, imputation and forecasting tasks. 
Other unsupervised tasks such as similarity search or clustering are naturally supported by extracting feature embeddings from \system.

\subsubsection{\textbf{Classification}\nopunct}\ \\
To classify timeseries, we input timeseries to the model as described in Sec.~\ref{sec.rita} and attach a special token \textbf{[CLS]} as the first input embedding. \textbf{[CLS]}'s embedding acts as the embedding for the entire timeseries, and the output representation of \textbf{[CLS]} is fed into a classifier: $\mathit{y=Softmax(W_{cls}Z_{[CLS]}+B_{cls})}$, where $Z_{[CLS]}\in \mathbb{R}^d$ is the output representation of \textbf{[CLS]}, C is the number of classes, and $\mathit{W_{cls} \in \mathbb{R}^{C \times d}, B_{cls} \in \mathbb{R}^{C}}$ are learnable parameters for classification task. 
The result vector $y\in \mathbb{R}^{C}$ represents the possibility that the input timeseries belongs to each class.

We apply Cross Entropy Loss as the loss function of the classification task~\cite{cox1958regression}:
$\mathit{L=\frac{1}{C}\sum_{i=1}^C -\hat{y}(i)log(y(i))}$, where $\hat{y}$ is a binary indicator for ground truth label:
\vspace{-0.5mm}
\begin{eqnarray}
\hat{y}(i) =
\begin{cases}
1   & i\  \text{is ground truth label} \\
0   & otherwise
\end{cases}
\end{eqnarray}

\subsubsection{\textbf{Imputation}\nopunct}\ \\
\label{sec.transformer.imputation}
Timeseries are mainly generated by sensors, a common problem of which is missing values. This becomes a challenge when many downstream analytics require the missing values to be recovered. The recovering task is imputation. 

Denote the real timeseries as $T_{r} \in \mathbb{R}^{t \times m}$, the observed timeseries with missing values as $T_{o} \in \mathbb{R}^{t \times m}$, and the set of missing values' positions as $M$. We scale the values of all timeseries to non-negative and use a special value (-1) to indicate missing values:
\vspace{-0.5mm}
\begin{eqnarray}
\label{eq.imp_task}
T_{o}(i,j) =
\begin{cases}
-1   & (i,j) \in M\\
T_{r}(i,j)   & (i,j) \notin M \\
\end{cases}
\end{eqnarray}

$T_{o}$ is fed into the \system as input, and the output representations are concatenated and fed into a {\it Transpose Convolution} layer which decodes the output embedding vectors from hidden space to timeseries values, corresponding to the convolution operation in the input stage, i.e., 
$\mathit{Y=TransposeCNN(Z_1 \textcircled{+} Z_2 \textcircled{+} ... \textcircled{+} Z_n)}$, where $Y \in \mathbb{R}^{t \times m}$ is the recovered timeseries, and $Z_i \in \mathbb{R}^d$ is the output of each position.

Here Mean Square Error is chosen as the loss function~\cite{thompson1990mse}:
$L=\frac{1}{|M|}\sum_{(i,j) \in M} (Y(i,j)-T_{r}(i,j))^2$.

\subsubsection{\textbf{Forecasting}\nopunct}\ \\
Forecasting can be regarded as a special case of imputation, in which all missing values are at the end of timeseries. 

So like in imputation task, we scale the timeseries to non-negative and use a special value (-1) to indicate the values to be predicted:
\begin{eqnarray}
T_{observed}(i,j) =
\begin{cases}
T_{real}(i,j)   & i \leq t_{observed} \\
-1   & otherwise
\end{cases}
\end{eqnarray}

Where $t_{observed}$ is the observed timestamp. Then the output representations are fed into a Transpose Convolution layer using Mean Squared Error as loss function, as described above.

\subsubsection{\textbf{Other Unsupervised Tasks}\nopunct}\ \\
\system naturally supports other unsupervised tasks, such as similarity search and clustering~\cite{lin1995fast,keogh2001dimensionality,liao2005clustering}, by producing the embedding of one timeseries (output representation of the special token \textbf{[CLS]}).
Clustering can be performed on the embeddings with flexible choice of distance metrics. Similarly, a high dimensional similarity search system~\cite{johnson2019billion, malkov2018efficient, jegou2010product} can be built on the embeddings.

\subsection{Inference Time}
\label{sec.sup.evaltime}

\begin{table}[htbp]
\centering
\small
\scalebox{0.88}{
\begin{tabular}{cc|c|c|c|c|c}
\toprule
Dataset &  Length & {TST\cite{DBLP:conf/kdd/ZerveasJPBE21}}& Vanilla & Performer & Linformer & Group Attn.\\
 \cline{3-7}
 \hline
 WISDM & 200 & 2.18 & 2.26 & 2.35 & 2.22 & 2.17\\
 HHAR & 200  & 1.19 & 1.23 & 1.28 & 1.21 & 1.18\\
 RWHAR & 200 & 1.32 & 1.37 & 1.42 & 1.34 & 1.31\\
ECG & 2000 & 18.44 & 15.26 & 5.80 & 6.08 & 5.16\\
 \bottomrule
\end{tabular}}
\caption{Inference time:  Classification on multi-variate data \quad \quad \quad (seconds).}
\label{tab.infcls}
\end{table}

\begin{table}[htbp]
\centering
\small
\scalebox{0.88}{
\begin{tabular}{cc|c|c|c|c|c}
\toprule
Dataset &  Length & 
 {TST\cite{DBLP:conf/kdd/ZerveasJPBE21}} & Vanilla & Performer & Linformer & Group Attn.\\
 \cline{3-7}
 \hline
 WISDM & 200 & 2.03 & 2.11 & 2.19 & 2.07 & 2.02\\
 HHAR & 200 & 1.11 & 1.14 & 1.19 & 1.12 & 1.10\\
 RWHAR & 200 & 1.23 & 1.27& 1.32& 1.25& 1.22 \\
ECG & 2000 & 17.22 & 14.32 & 4.73 & 4.99 & 4.11\\
MGH & 10000 & N/A & N/A & 6.58 & 6.88 & 1.35\\
 \bottomrule
\end{tabular}}
\caption{Inference time: Imputation on multi-variate data (seconds). }
\label{tab.infimp}
\vspace{-5mm}
\end{table}

In this section, we present the average inference time on validation sets. The results in Table.~\ref{tab.infcls} and ~\ref{tab.infimp} correspond to the average inference time on validation sets of classification and imputation tasks, respectively. Consistent with the results in Section.~\ref{sec.exp.efficiency}, our method Group Attn. outperforms the baselines on both classification and imputation tasks, particularly on the datasets comprising long timeseries (ECG and MGH).

\end{sloppypar}

\end{document}